%% file: main.tex
\documentclass{article}

\DeclareMathAlphabet{\mymathbb}{U}{BOONDOX-ds}{m}{n}
\usepackage{xcolor}

\usepackage{xcolor}
\usepackage{amsthm}
\usepackage{framed}
\theoremstyle{plain}% default

\newtheorem{definition}{Definition}[section]
\newtheorem{theorem}{Theorem}[section]
\newtheorem*{theorem*}{Theorem}
\newtheorem{proposition}{Proposition}[section]
\newtheorem{corollary}{Corollary}[theorem]
\newtheorem{lemma}[theorem]{Lemma}

\theoremstyle{remark}
\newtheorem{remark}{Remark}

\newcommand{\gr}{\nabla}

\newcommand{\E}{\mathbb{E}}
\newcommand{\Prob}{\mathbb{P}}

\newcommand{\ubar}{\overline{u}}
\newcommand{\vbar}{\overline{v}}

\newcommand{\SubE}{\normalfont{\mathsf{SubE}}}
\newcommand{\SNR}{{\mathsf{snr}}}
\newcommand{\SigmaHat}{\widehat{\Sigma}}
\newcommand{\vHat}{\widehat{v}}

\newcommand{\op}{\operatorname{op}}

\usepackage[T1]{fontenc}
\usepackage[cmex10]{amsmath}
\usepackage[utf8]{inputenc}
\usepackage{pgfplots}
\makeatletter
\def\BState{\State\hskip-\ALG@thistlm}
\makeatother

\usepackage[T1]{fontenc}
\usepackage[english]{babel}
\usepackage[utf8]{inputenc}
\usepackage{pgfplots}
\usepackage{graphicx}
\usepackage{subcaption}
\usepackage{pstricks}
\usepackage{color}
\usepackage{amsmath}
\usepackage{bbm}
\usepackage{tcolorbox}
\usepackage{tikz}
\usepackage{pgfplots}
\usepackage{wrapfig}
\usepackage{lipsum}  
\usetikzlibrary{positioning}
\usepackage{tcolorbox}
\usepackage{amsfonts,amssymb}

\usepackage[round]{natbib}
\usepackage{mathtools}
\usepackage{commath}
\usepackage{relsize}
\usepackage{bbm}
\usepackage{bm}
\usepackage[font={small}]{caption}
\usepackage{comment,color,soul}
\usepackage{enumerate} 
\usetikzlibrary{arrows,automata}
\usepackage{amsfonts}
\usepackage{url}
\usepackage{lipsum}
\usepackage[thinlines]{easytable}
\usepackage{nicefrac}
\usepackage{algorithm}
\usepackage{algpseudocode}
\usepackage{mysymbol}
\usepackage{float}
\usepackage{comment}
\usepackage{scalerel}

\usepackage[top=2.5cm,bottom=3.5cm,left=2.5cm,right=2.5cm,marginparwidth=1.75cm]{geometry}

 \makeatletter
\def\BState{\State\hskip-\ALG@thistlm}
\makeatother

\makeatletter
\newcommand{\algmargin}{\the\ALG@thistlm}
\makeatother

\usepackage{amsmath}
\mathtoolsset{showonlyrefs}
\usepackage{etoolbox}

\makeatletter
\newif\ifdavid@number
\preto\equation{\david@numberfalse}
\preto\endequation{\ifdavid@number\else\notag\fi}
\patchcmd\label@in@display{\@empty}{\@empty\david@numbertrue}{}{}
\makeatother

\makeatletter
\def\Let@{\def\\{\notag\math@cr}}
\makeatother
\usepackage[hidelinks]{hyperref}
\definecolor{darkred}{RGB}{150,0,0}
\definecolor{darkgreen}{RGB}{0,150,0}
\definecolor{darkblue}{RGB}{0,0,150}
\hypersetup{colorlinks=true, linkcolor=darkred, citecolor=darkblue, urlcolor=darkblue}

\usepackage{enumitem}
\usepackage{amsfonts} 
\usepackage{tcolorbox}
\mathtoolsset{showonlyrefs}
\usepackage{float}
\usepackage{comment}
\usepackage{scalerel}
\usepackage{breqn}

\usepackage{tabulary,booktabs}
\usepackage{multirow}

\usepackage{pgfplots}
\usetikzlibrary{calc}

\usepackage{hhline}
\usepackage{tabu}

\title{\textbf{EM for Mixture of Linear Regression with Clustered Data}}

\date{}

\makeatletter
\renewcommand*{\@fnsymbol}[1]{\ensuremath{\ifcase#1\or 1 \or 2 \or 3 \or 4 \else\@ctrerr\fi}}
\makeatother

\author{
  Amirhossein Reisizadeh\\
  \texttt{amirr@mit.edu}
  \and 
  Khashayar Gatmiry\\
  \texttt{gatmiry@mit.edu}
  \and Asuman Ozdaglar\\
  \texttt{asu@mit.edu}
}

\date{Massachusetts Institute of Technology}

\begin{document}

\maketitle

\begin{abstract}
    \input{0-abstract}

\end{abstract}

\section{Introduction}

\input{1-intro}

\section{Preliminaries}

\input{1-setting}

\section{Analysis of Population and Empirical EM Updates}

\input{2-optimization}

\subsection{Empirical EM update}

\input{3-generalization}

\section{Main Results on Sample-based EM Algorithm}

\input{8-convergence}

\section{Conclusion}

\input{9-conclusion}

\section*{Acknowledgments}
This work was supported, in part, by MIT-DSTA grant 031017-00016 and the MIT-IBM Watson
AI Lab.

\bibliographystyle{plainnat}
\bibliography{biblio}

\clearpage
\newpage

\appendix

\section{Proof of Theorem \ref{thm: FOS}} \label{sec: proof thm FOS}

\input{5-opt-proof}

\section{Proof of Theorem \ref{thm: generalization}} \label{proof: thm generalization}

\input{4-gen-proof}

\section{Useful Lemmas}

\input{6-useful_stuff}

\end{document}

%% file: 0-abstract.tex
Modern data-driven and distributed learning frameworks deal with diverse massive data generated by clients spread across heterogeneous environments. Indeed, \emph{data heterogeneity} is a major bottleneck in scaling up many distributed learning paradigms. In many settings however, heterogeneous data may be generated in \emph{clusters} with shared structures, as is the case in several applications such as federated learning where a common latent variable governs the distribution of all the samples generated by a client. It is therefore natural to ask how the underlying clustered structures in distributed data can be exploited to improve learning schemes. In this paper, we tackle this question in the special case of estimating  $d$-dimensional parameters of a two-component mixture of linear regressions problem where each of $m$ nodes generates $n$ samples with a \emph{shared} latent variable. We employ the well-known Expectation-Maximization (EM) method to estimate the maximum likelihood parameters from $m$ batches of dependent samples each containing $n$ measurements. Discarding the clustered structure in the mixture model, EM is known to require $\mathcal{O}(\log(mn/d))$ iterations to reach the statistical accuracy of $\mathcal{O}(\sqrt{d/(mn)})$. In contrast, we show that if initialized properly, EM on the structured data requires only $\mathcal{O}(1)$ iterations to reach the same statistical accuracy, as long as $m$ grows up as $e^{o(n)}$. Our analysis establishes and combines novel asymptotic optimization and generalization guarantees for population and empirical EM with dependent samples, which may be of independent interest.

%% file: 1-intro.tex
With the ever-growing applications of data-intensive and distributed learning paradigms, it becomes more critical to address new challenges associated with such frameworks. For instance, federated learning is a novel distributed learning architecture consisting a central parameter server and a network of clients (or nodes) each equipped with locally generated data. In general, the main premise of such distributed learning methods is to estimate the underlying ground truth model using the collective data samples across the clients. \emph{Data heterogeneity} (or non-i.i.d. data) is among the most significant challenges in scaling up distributed learning methods. Indeed, naive distributed and federated benchmarks such as FedAvg are known to diverge if deployed on highly heterogeneous settings, unless particularly tailored for non-i.i.d. data \citep{karimireddy2020scaffold}.

In this paper, we consider a \emph{structured} or \emph{clustered} data heterogeneity model which roots in an observation specific to modern data-driven distributed and federated learning applications. Under this structured heterogeneity model, an \emph{identical} and unobserved latent  variable governs the distribution of \emph{all} the samples generated at any node \citep{pei2017deepxplore,hendrycks2019benchmarking,robey2020model,diamandis2021wasserstein}. Particularly in this paper, we zoom in on \emph{mixture of linear regression} model which is a classical approach to capture data heterogeneity \citep{jordan1994hierarchical,xu2016global,viele2002modeling}. To be more clear, in our setting each node observes not one but a potentially large number of linear measurements for all of which a common latent variable governs the true parameter. These latent variables are unknown, random, independent and identically distributed across the nodes. Throughout the paper, we refer to this model as \emph{clustered mixture of linear regressions}, or C-MLR in short.

Our goal in this work is to estimate the maximum likelihood parameters of the regression model in the above-described C-MLR heterogeneity model using the collection of \emph{all} the observations across all the devices.  However, maximizing likelihood objectives are notoriously intractable in general, due to non-convexity of the likelihood function \citep{yi2014alternating}. The most popular approach for computationally efficient inference in such models with latent variables is the Expected-Maximization (EM) method \citep{dempster1977maximum,redner1984mixture,wu1983convergence}. We therefore aim to study optimization and generalization characteristics of the EM method in estimating the C-MLR models. 

To this end, we first characterize and analyse the so-called \emph{population EM} variant for which we establish an asymptotic, local and deterministic convergence guarantee. Next, we move to the empirical counterpart with finite number of observations known as the \emph{empirical EM} method and provide probabilistic generalization bounds on its estimation error. Both results are local and asymptotic. That is, our analysis relies on the assumption that the initial iterate of the EM method is suitable (as opposed to random). Moreover, we let the number of nodes and the number of samples per node grow while all the other parameters assumed to be constants.  To be more specific, let us precisely describe the C-MLR model in the following.

\subsection{Clustered MLR model}

As discussed above and motivated by distributed learning applications, we consider a collection of $m$ nodes  where each node $j =1,\cdots,m$ observes $n$ pairs of measurements denoted by $\{(x_i^j, y_i^j) | i=1,\cdots,n\}$. Here, $x_i^j \in \ccalX \subseteq \reals^d$ and $y_i^j \in \ccalY \subseteq \reals$ denote the covariate and response variables, respectively. These observations are linear measurements of a \emph{clustered} mixture of linear regressions (C-MLR) model described below
%%%%%
\begin{flalign} 
    &&\hspace{2cm} y_i^j = \xi^j \langle x_i^j, \theta^* \rangle + \epsilon_i^j,
    \quad
    i=1,\cdots,n, \quad j = 1,\cdots,m.
    &&\text{{\normalfont (C-MLR)}}\quad&
    \label{eq: emp-C-MLR model}
\end{flalign}
%%%%%
In this model, $\xi^j \in \Xi$ denotes the hidden latent variable corresponding to node $j$. In this paper, we focus on a symmetric and two-component mixture of linear regressions with $\Xi = \{-1, +1\}$, where $\xi^j$ takes on values uniformly at random, denoted by $\xi^j \sim \ccalU\{\pm 1\}$. Note that this latent variable is \emph{identical} for \emph{all} the measurements of a given node, however, we assume that they are \emph{independent} across different nodes. Moreover, we let $\theta^* \in \reals^d$ denote the fixed and unknown ground truth regression vector and assume that covariates and noises are independent and Gaussian with $x_i^j \sim \ccalN(0, I_d)$ and $\epsilon_i^j \sim \ccalN(0, \sigma^2)$, respectively. This model clearly implies that the observations of any given node are \emph{not} independent due to the shared latent variable. In the remainder of the paper, we denote the signal-to-noise ratio (SNR) by $\SNR = \| \theta^* \| / \sigma$.

%%%%
\begin{remark}
C-MLR model in \eqref{eq: emp-C-MLR model} captures the underlying  node-dependent data heterogeneity through the  latent variable $\xi^j$ which is shared and identical for all the $n$ samples measured by node $j$. Therefore, C-MLR is a well-motivated abstract model to encapsulate the structured data heterogeneity observed in modern distributed learning application as discussed before \citep{diamandis2021wasserstein}.
\end{remark}
%%%%

%%%%
\begin{remark}
We further clarify that in the C-MLR model described above, the term ``clustered'' referrers to the fact that data samples are available in batches of size $n$ where all the $n$ samples in each batch share the same latent variable $\xi$. Though, it is worth noting that the folklore two-component MLR model with independent latent variables partitions the samples into two clusters, as well. However, we adopt the term ``clustered'' to particularly underscore the batched structure modeled in \eqref{eq: emp-C-MLR model}.
\end{remark}
%%%%

%%%%
\begin{remark}
In our asymptotic analysis in this paper, we are interested in the regime that $m$ and $n$ grow while other problem parameters, that are $\| \theta^* \|$, $\sigma$, and $d$ remain constant.
\end{remark}
%%%%
 \newpage
 
Our main goal is this paper is to answer he following question:
%%%%
\begin{tcolorbox}
\begin{center}
\textit{What is iteration complexity of the sample-based EM algorithm to estimate the ground truth $\theta^*$ from $m$ batches of samples, each of size $n$ generated by the C-MLR described in \eqref{eq: emp-C-MLR model}?}
\end{center}
\end{tcolorbox}
%%%%

We answer this question in this paper as follows. We assume that $m$ batches of in total $mn$ samples generated by the C-MLR model in \eqref{eq: emp-C-MLR model} are available where $m$ grows at most up to $e^{o(n)}$. We prove that if initialized within a constant-size neighbourhood of the ground truth $\theta^*$ and after $T=\ccalO(1)$ iterations of the sample-based (or empirical) EM algorithm, either (\emph{i}) there exists an iterate $0 \leq t \leq T$ of the algorithm for which $\| \theta_t - \theta^* \| \leq \ccalO(\sqrt{d/(mn)})$; or (\emph{ii}) the $\| \theta_T - \theta^* \| \leq \ccalO(\sqrt{d/(mn)})$ with high probability. Our result is asymptotic, that is, it holds for sufficiently large $n$. To highlight this result, it is worth noting that the underlying clustered structure in C-MLR is essential for a constant iteration complexity. Indeed, if such structure is discarded, the EM algorithm requires $\ccalO(\log(mn/d))$ iterates to reach the same statistical accuracy.

\textbf{Contribution.} To summarize the above discussion, ee consider a data heterogeneity structure observed in various distributed learning application such as federated learning where a latent variable governs the distribution of all the samples generated on any node. In particular, we zoom in on a \emph{clustered} two-component mixture of  linear regression model described in \eqref{eq: emp-C-MLR model} where all the linear measurements of any node share their binary latent variable. We utilize the EM algorithm to estimate the maximum likelihood regressor and establish asymptotic and local optimization and generalization guarantees for both population and empirical EM updates. Lastly, we employ these two results and asymptotically characterize the iteration complexity of the sample-based EM algorithm to estimate the ground truth parameters of the C-MLR model.

\textbf{Related work.} Studying convergence characteristics of Expectation-Maximization (EM)  dates back to the seminal work of \cite{wu1983convergence} in which asymptotic and local convergence of EM is established for general latent variable models. \cite{balakrishnan2017statistical} provides a general framework to analyze local onvergence of the EM algorithm in several settings such as mixture of linear regressions (MLR) and Gaussian mixture model (GMM). Several follow up works study GMM, MLR and Missing Covariate Regression (MCR) models including \cite{yi2015regularized,daskalakis2017ten,li2018learning,klusowski2019estimating,ghosh2020alternating,yan2017convergence}.

Although it is not the main focus of this paper, global convergence of the EM method (with random initialization) has been extensively studied for Gaussian mixture model \citep{chen2019gradient} and mixture of linear regressions \citep{kwon2019global,wu2019randomly}. Another interesting direction is establishing statistical lower bounds on the accuracy of the EM method for the MLR model \cite{kwon2021minimax}. Going beyond two-component MLR model, \cite{kwon2020converges} proves that well-initialized EM converges to the true regression parameters of $k$-component MLR in certain SNR regimes. In the same setting, \cite{chen2020learning} proposes an algorithm that is sub-exponential in $k$. For noiseless MLR model, \cite{yi2014alternating,yi2016solving} were among the first works to establish convergence guarantees for EM. To tackle the computational complexity of EM in learning MLR models, \cite{li2018learning,zhong2016mixed} propose gradient descent-type methods with nearly optimal sample complexity. From practical point of view, EM has demonstrated empirical success in MLR models \citep{jordan1994hierarchical,de1989mixtures} and its simple implementation has made it a suitable choice in several applications \citep{chen2009hypothesis,li2009non}.

%% file: 1-setting.tex
In this section, we first review backgrounds on MLE and EM and then characterize the population and empirical EM updates for our C-MLR model followed by an insightful benchmark.

%%%%%%%%%%%%%%%%%%%%%%%%%%%%%%%%%%%%
\subsection{Maximum Likelihood Estimator and EM Algorithm}
%%%%%%%%%%%%%%%%%%%%%%%%%%%%%%%%%%%%

\textbf{Population EM.} Let us focus on one node observing $n$ samples $\{(x_i,y_i) | i =1,\cdots,n\}$ where we adopt the shorthand notations $x_{[n]} = (x_1, \cdots, x_n)$  and $y_{[n]} = (y_1, \cdots, y_n)$. Furthermore, let $\xi$ denote the latent variables in the C-MLR model described in \eqref{eq: emp-C-MLR model}, respectively. To reiterate the underlying C-MLR model, we have that
%%%%
\begin{align} \label{eq: pop-C-MLR model}
    y_i = \xi \langle x_i, \theta^* \rangle + \epsilon_i,
    \quad
    i=1,\cdots,n.
\end{align}
%%%%
As discussed before, in our setting, only the variables $(x_{[n]},y_{[n]})$ are observed and the latent variable $\xi \in \Xi$ remains hidden. Suppose that the tuple $(x_{[n]},y_{[n]},\xi)$ is generated by the joint distribution $f_{\theta^*}$ where $\{f_{\theta} | \theta \in \Omega\}$ and $\Omega$ is a non-empty compact convex set.

As our main goal in this paper, we aim to estimate the ground-truth model $\theta^*$ by maximizing the likelihood function, that is, finding $\hat{\theta} \in \Omega$ that maximizes the following likelihood
%%%%
\begin{align}
    g_{\theta}(x_{[n]},y_{[n]})
    =
    \int_{\Xi} f_{\theta}(x_{[n]},y_{[n]},\xi) \mathrm{d} \xi.
\end{align}
%%%%
In many settings, it is computationally expensive to compute the likelihood function $g_{\theta}(x_{[n]},y_{[n]})$, while computing log-likelihood $\log f_{\theta}(x_{[n]},y_{[n]},\xi)$ is relatively easier. The EM method is an iterative algorithm that aims to maximize a lower bound on the log-likelihood $\log g_{\theta}(\cdot,\cdot)$. This lower bound which is known as the $Q$-function can be written as follows
%%%%
\begin{gather} \label{eq: pop-Q}
    Q(\theta' | \theta)
    =
    \int_{\ccalX^n \times \ccalY^n} \bigg( \int_{\Xi} f_{\theta} (\xi | x_{[n]},y_{[n]}) \log f_{\theta'} (x_{[n]},y_{[n]},\xi) \mathrm{d} \xi \bigg) f_{\theta^*}(x_{[n]},y_{[n]}) \mathrm{d}x_{[n]} \mathrm{d}y_{[n]}. 
\end{gather}
%%%%
At each iteration of the empirical EM (Algorithm \ref{alg:pop-EM}) and given the current estimate of the true model $\theta$, the next model is obtained by maximizing the above $Q$-function, that is, $\theta \gets M(\theta)$ where
%%%%
\begin{gather} 
    M(\theta)
    \coloneqq
    \argmax_{\theta' \in \Omega} Q(\theta' | \theta).\label{eq: pop-EM}
\end{gather}
%%%%
Note that computing $M(\cdot)$ requires having access to the joint distribution $f_{\theta^*}$, or to put it differently, observed data from infinitely many nodes $(m \to \infty)$ is required. We call such variant of the EM algorithm \emph{population EM} and discuss the \emph{empirical} variant with finite clients (finite $m$) in the following section. Next proposition characterizes the $M$-function and the population EM update.

%%%%
\begin{proposition}[Population EM] \label{prop: population EM}
Consider $n$ linear measurements from the C-MLR model in \eqref{eq: pop-C-MLR model} with Gaussian features $X_i \sim \ccalN(0,I_d)$ and noises $\epsilon_i \sim \ccalN(0,\sigma^2)$ with shared latent variable $\xi \sim \ccalU\{\pm 1\}$. Then, the $M(\cdot)$ function of the population EM defined in \eqref{eq: pop-EM} is as follows
%%%%
\begin{align} \label{eq: pop M}
    M(\theta)
    &=
    \E \bigg[X_1 Y_1 \tanh\bigg( \frac{1}{\sigma^2} \sum_{i=1}^{n} \langle X_i, \theta\rangle Y_i \bigg) \bigg].
\end{align}
%%%%
\end{proposition}
%%%%
%%%%
\begin{proof}
We defer the proof to Appendix \ref{proof: prop population EM}.
\end{proof}
%%%%

Note that equally likely $\xi \in \{\pm 1\}$ makes the distribution of $Y^n$ symmetric given $X^n$. Moreover, $\tanh(\cdot)$ is an odd function and therefore, the expectation in \eqref{eq: pop M} can also be taken with respect to $X_i \sim \ccalN(0,I_d)$ and $Y_i | X_i \sim \ccalN(\langle X_i,\theta^* \rangle, \sigma^2)$, \emph{i.e.} no randomness in the the latent variable $\xi$.

\textbf{Empirical EM.} For a finite number of nodes $m$, the empirical EM algorithm updates the estimate of the true model using the empirical $Q_m$-function defined below
%%%%
\begin{align} \label{eq: emp Q}
    Q_m(\theta' | \theta)
    =
    \frac{1}{m} \sum_{j=1}^{m}\int_{\Xi} f_{\theta} (\xi | x^j_{[n]},y^j_{[n]}) \log f_{\theta'} (x^j_{[n]},y^j_{[n]},\xi) \mathrm{d} \xi,
\end{align}
%%%%
where samples are independent across different nodes. Similarly, in each iteration of the empirical EM algorithm (Algorithm \ref{alg:emp-EM}), the current model estimate $\theta$ is updated to $\theta \gets M_m(\theta)$ where 
%%%%
\begin{align} \label{eq: emp-EM}
    M_m(\theta)
    \coloneqq
    \argmax_{\theta' \in \Omega} Q_m(\theta' | \theta).
\end{align}
%%%%
Next proposition characterizes the empirical $M_m$-function defined in \eqref{eq: emp-EM}.

%%%%
\begin{proposition}[Empirical EM] \label{prop: empirical EM}
Consider $m$ nodes each observing $n$ linear measurements generated by the C-MLR model in \eqref{eq: emp-C-MLR model} denoted by $\{(x_i^j, y_i^j) | i =1,\cdots,n, \, j =1,\cdots,m\}$. Then, the $M_m(\cdot)$ function of the empirical EM defined in \eqref{eq: emp-EM} can be computed as follows
%%%%
\begin{align} \label{eq: emp M}
    M_m(\theta)
    &=
    \SigmaHat^{-1} \frac{1}{mn} \sum_{j=1}^{m} \sum_{i=1}^{n} x_i^j y_i^j \tanh\bigg( \frac{1}{\sigma^2} \sum_{i=1}^{n} \langle x_i^j, \theta\rangle y_i^j \bigg),
    \, \text{ where } \,
    \SigmaHat
    \coloneqq
    \frac{1}{mn} \sum_{j=1}^{m} \sum_{i=1}^{n} x_i^j {x_i^j}^{\top}
\end{align}
%%%%
denotes the sample covariance matrix of the total $mn$ observations.
\end{proposition}
%%%%
%%%%
\begin{proof}
We defer the proof to Appendix \ref{proof: prop populatiempiricalon EM}.
\end{proof}
%%%%

\begin{figure}
\begin{minipage}[t]{0.48\textwidth}
\begin{algorithm}[H]
\caption{Population EM}\label{alg:pop-EM}
\begin{algorithmic} 
\Require initialization $\theta_0$
\For{$t=0,1,\cdots$}
\State Update $\theta_{t+1} = M(\theta_t)$ as defined in \eqref{eq: pop-EM}
\EndFor
\end{algorithmic}
\end{algorithm}
\end{minipage}
\hfill
\begin{minipage}[t]{0.48\textwidth}
\begin{algorithm}[H]
\caption{Empirical EM}\label{alg:emp-EM}
\begin{algorithmic}
\Require initialization $\theta_0$
\For{$t=0,1,\cdots$}
\State Update $\theta_{t+1} = M_m(\theta_t)$ as defined in \eqref{eq: emp-EM}
\EndFor
\end{algorithmic}
\end{algorithm}
\end{minipage}
\end{figure}

Our goal in the remainder of the paper is to rigorously study the optimization and generalization performance of the two population and empirical EM algorithms described above. Before that, let us elaborate on a simple and intuitive benchmark.

\subsection{A benchmark: EM with independent samples} \label{sec: benchmark}

As we described in our C-MLR model in \eqref{eq: emp-C-MLR model}, the measurements observed on a given node share the same latent variable, making them dependent. In contrast, the well-established literature on EM is centered around the i.i.d. setting where each sample is generated through a latent variable independent of the ones for any other sample. To be more precise, consider the setting where $N$ i.i.d. linear measurements $\{(x_i,y_i) | i=1,\cdots,N\}$ generated by a mixture of two component linear regression model are available. That is, $y_i = \xi_i \langle x_i, \theta^* \rangle + \epsilon_i$ for all $i=1,\cdots,N$ where $\xi_i \sim \ccalU\{\pm 1\}$,  $x_i \sim \ccalN(0,I_d)$ and $\epsilon_i \sim \ccalN(0,\sigma^2)$ are i.i.d. and mutually independent. In this setting, the population and empirical EM update rules are as follows 
%%%%
\begin{align} \label{eq: pop M iid}
    M(\theta)
    =
    \E \Big[X Y \tanh\Big( \frac{1}{\sigma^2} \langle X, \theta\rangle Y \Big) \Big],
    \, \text{ and } \,
    M_N(\theta)
    =
    \SigmaHat^{-1} \frac{1}{N} \sum_{i=1}^{N} x_i y_i \tanh\Big( \frac{1}{\sigma^2} \langle x_i, \theta \rangle y_i \Big),
\end{align}
%%%%
where the expectation is over $X \sim \ccalN(0,I_d)$, $\xi \sim \ccalU\{\pm 1\}$ and $Y | X,\xi \sim \ccalN(\xi \langle X,\theta^* \rangle, \sigma^2)$. In above, $\SigmaHat = 1/N \sum_{i=1}^{N} x_i x_i^{\top}$ denotes the sample covariance matrix \citep{balakrishnan2017statistical,kwon2019global}. In particular, it was shown in \cite{balakrishnan2017statistical} that for any suitable initialization with $\| \theta_0 - \theta^* \| \leq \| \theta^* \|/32$, after $T = \log(N/d \cdot \| \theta^* \|^2/(\| \theta^* \|^2 + \sigma^2)) \cdot \ccalO(1)$ iterations of empirical EM with update rule $M_N(\cdot)$ as above, the following sub-optimality is guaranteed with probability at least $1-\delta$,
%%%%
\begin{align}
    \| \theta_T - \theta^* \| 
    \leq
    \sqrt{\| \theta^* \|^2 + \sigma^2} \, \sqrt{\frac{d+\log(1/\delta)}{N}}  \, \log \bigg( \frac{N}{d} \cdot \frac{\| \theta^* \|^2}{\| \theta^* \|^2 + \sigma^2}  \bigg) \cdot \ccalO(1).
\end{align}
%%%%
Now, consider $N = mn$ linear measurements generated by the C-MLR model in \eqref{eq: emp-C-MLR model} which we also denote by the same notation $\{(x_i,y_i) | i=1,\cdots,N\}$. Clearly, the EM update rules in \eqref{eq: pop M iid} may not be employed in this setting as samples are not independent due to the shared latent variables. However, one could make such $N$ samples independent by the following simple trick. For each sample $i=1,\cdots,N$, let us denote $\tilde{y}_i=\tilde{\xi}_i \cdot y_i$ where $\tilde{\xi}_i$s are independent Rademacher variables. In words, $\tilde{y}_i=y_i$ or $\tilde{y}_i=-y_i$ equally likely. It is straightforward to check that the new $N$ samples $\{(x_i,\tilde{y}_i) | i=1,\cdots,N\}$ are indeed independent. Therefore, one may employ the guarantee above and conclude that with a suitable initialization and after $T$ iterations of EM (on the new samples), the final sub-optimality is with probability $1-\delta$ bounded by
%%%%
\begin{align}
    \| \theta_T - \theta^* \| 
    \leq
    \sqrt{\| \theta^* \|^2 + \sigma^2} \, \sqrt{\frac{d + \log(1/\delta)}{mn}} \, \cdot \tilde{\ccalO}(1),
    \, \text{ where } \,
    T =  \log \Big(\frac{mn}{d} \cdot \frac{\| \theta^* \|^2}{\| \theta^* \|^2 + \sigma^2} \Big) \cdot \ccalO(1).
\end{align}
%%%%

As mentioned before, we aim to characterize the complexity of the EM algorithm deployed on clustered samples per the C-MLR model described in \eqref{eq: emp-C-MLR model}. Before laying out our formal analysis, it is worth highlighting our main result here and comparing it to the simple benchmark described above.

%%%%
\begin{theorem*}[Main, informal] 
Consider the empirical EM in Algorithm \ref{alg:emp-EM} with a constant $\SNR \geq 4$ and any tolerance probability $\delta \in (0,1)$. Moreover, assume that $mn \geq \ccalO(d + \log(1/\delta))$ and $n \geq \ccalO(\log(m) + d + \log(1/\delta))$. Then, for a suitable initialization and sufficiently large $n$, after $T=\ccalO(1)$ iterations of Algorithm \ref{alg:emp-EM}, either
%%%%
\begin{enumerate}[label=(\roman*),nosep]
    \item there exists an iterate $0 \leq t \leq T$ such that
%%%%
\begin{align}
    \| \theta_t - \theta^* \| 
    \leq
    \sqrt{\| \theta^* \|^2 + \sigma^2}  \sqrt{\frac{d + \log(1/\delta)}{mn}},
\end{align}
%%%%
\item or with probability at least $1-\delta$,
%%%%
\begin{align}
    \| \theta_T - \theta^* \| 
    \leq
    \sqrt{\| \theta^* \|^2 + \sigma^2}  \sqrt{\frac{d + \log(1/\delta)}{mn}} \cdot \ccalO(1).
\end{align}
%%%%
\end{enumerate}
\vspace{-5pt}
\end{theorem*}
%%%%
Our result above demonstrates that incorporating the underlying clustered structure in the C-MLR model, EM requires only $\ccalO(1)$ iterations to reach the statistical accuracy  $\ccalO(\sqrt{d/(mn)})$ under proper scaling assumptions. In contrast and as illustrated above, discarding such structure makes EM algorithm to run for $\ccalO(\log(mn/d))$ iterations to reach the same accuracy.

In the following sections, we prove this result by laying out optimization and generalization guarantees for the EM algorithm on samples generated by the C-MLR model.

%% file: 2-optimization.tex
\subsection{Population EM update}

In this section, we consider the population EM updates in Algorithm \ref{alg:pop-EM} with the $M$ operator characterized in \eqref{eq: pop M} and establish optimization guarantees for it. Let us recall the population EM scenario and the underlying C-MLR model. Denoted by $\{(x_i,y_i) | i =1,\cdots,n\}$ are $n$ pairs of linear measurements generated according to the mixture model \eqref{eq: pop-C-MLR model}, that is, $y_i = \xi \langle x_i, \theta^* \rangle + \epsilon_i$ for all $i = 1,\cdots,n$. In Proposition \ref{prop: population EM}, we characterised the population $M$-function and in the following theorem, we establish its contraction property. Here and throughout the paper, we denote a Euclidean ball of radius $r$ around the fixed point $\theta^*$ by $\mymathbb{B}(r;\theta^*) \coloneqq \{ \theta \in \Omega \, | \, \| \theta - \theta^* \| \leq r\}$.
%%%%
\begin{theorem} \label{thm: FOS}
Consider the population EM update rule $M$ in \eqref{eq: pop M} and assume that $\theta \in \mymathbb{B}(\alpha \| \theta^* \|; \theta^*)$ for some constant $0 \leq \alpha < 1$. If  $\| \theta - \theta^*\| \geq \varepsilon$, then there exist constants $N_0(\alpha, \SNR)$ and $C(\alpha, \SNR)$ depending on $\alpha$ and $\SNR=\| \theta^* \|/\sigma$ such that for any $n \geq N_0(\alpha, \SNR)$ we have
%%%%
\begin{align}
    \Vert M(\theta) - \theta^* \Vert
    \leq
    \kappa \Vert \theta - \theta^* \Vert,
    \quad \text{for} \quad
    \kappa
    =
    \big(\| \theta^* \| + \sigma \big)
    \Big( \SNR + \frac{1}{n \varepsilon} \Big) \exp \left( -n \cdot C(\alpha, \SNR) \right).
\end{align}
%%%%
\end{theorem}
%%%%

%%%%
\begin{proof}
We defer the proof to Appendix \ref{sec: proof thm FOS}.
\end{proof}
%%%%

The result of this theorem reveals a number of insightful remarks as follows.

\begin{remark}
First, for any constant accuracy lower bound $\varepsilon$, as the number of samples per node $n$ grows, the factor $\kappa$ decreases and there exists a constant $N_0$ depending on the problem parameters such that for any $n \geq N_0$, the $M$-operator is a contraction, that is, $\kappa < 1$. Secondly and more importantly, it shows that if initialized within a ball around the ground truth model $\theta^*$, iterates of the population EM in Algorithm \ref{alg:pop-EM} converge \emph{linearly} in $n$ till reaching the accuracy $\varepsilon$. The following corollary provides an informal but insightful implication of this theorem.
\end{remark}

%%%%
\begin{corollary} [Informal]
Suppose that the population EM in Algorithm \ref{alg:pop-EM} is initialized with $\theta_0$ where $\| \theta_0 - \theta^*\| = \ccalO(\| \theta^* \|)$. Then, for sufficiently large $n$ and after  $T = \ccalO(1 + \log(n/d)/n) = \ccalO(1)$ iterations, either there exists an iterate $0\leq t \leq T$ for which $\| \theta_t - \theta^* \| = \ccalO(\sqrt{d/n} \, \| \theta^* \|)$.
\end{corollary}
%%%%

While we provide the proof of Theorem \ref{thm: FOS} in Section \ref{sec: proof thm FOS}, it is worth elaborating on the proof technique as follows.

\subsection{Proof sketch} 

To establish optimization guarantees for the population EM iterates and Algorithm \ref{alg:pop-EM}, we first adopt the \emph{First-Order Stability} (FOS) notion \citep{balakrishnan2017statistical} as defined below. 
%%%%
\begin{definition}[First-Order Stability (FOS)] \label{def: FOS}
The functions $\{Q(\cdot | \theta) | \theta \in \Omega\}$ satisfy condition FOS($\gamma$) over $\mymathbb{B}(r;\theta^*)$ if
%%%%
\begin{align}
    \Vert \gr Q(M(\theta)|\theta^*) - \gr Q(M(\theta)|\theta) \Vert
    \leq
    \gamma \Vert \theta - \theta^* \Vert,
    \quad
    \text{for all } \theta \in \mymathbb{B}(r;\theta^*).
\end{align}
%%%%
\end{definition}
%%%%
This property of the $Q$-function helps showing the contraction of the population EM operator $M$. The following general theorem from \cite{balakrishnan2017statistical} characterizes the conditions under which the population EM operator $M$ is contractive.
%%%%
\begin{theorem}[\cite{balakrishnan2017statistical}] \label{thm: M contractive}
For some radius $r > 0$ and pair $(\gamma, \lambda)$ such that $0 \leq \gamma < \lambda$, suppose that the
function $Q(\cdot|\theta^*)$ is $\lambda$-strongly concave, and that the FOS($\gamma$) condition holds on the
ball $\mymathbb{B}(r;\theta^*)$. Then, the population EM operator $M$ is contractive over $\mymathbb{B}(r;\theta^*)$, in particular,
%%%%
\begin{align}
    \Vert M(\theta) - \theta^* \Vert
    \leq
    \frac{\gamma}{\lambda} \Vert \theta - \theta^* \Vert,
    \quad
    \text{for all } \theta \in \mymathbb{B}(r;\theta^*).
\end{align}
%%%%
\end{theorem}
%%%%

For the EM function in \eqref{eq: pop M}, we prove the first-order stability property in Definition \ref{def: FOS} for a fixed $\theta$. More precisely, for any $\theta \in \mymathbb{B}(\alpha \| \theta^* \|; \theta^*)$, we show that  for the population $Q$-function \eqref{eq: pop-Q} the FOS($\gamma$) property holds true with
%%%%
\begin{align}
    \gamma
    =
    \frac{1}{\sigma^2}\big( \| \theta^* \| + \sigma \big)
    \Big( n \cdot \SNR + \frac{1}{\varepsilon} \Big) \exp \left( -n \cdot C(\alpha, \SNR) \right),
\end{align}
%%%%
as long as $\| \theta - \theta^*\| \geq \varepsilon$. On the other hand, it is straightforward to check that population $Q$-function is $\lambda$-strongly concave with $\lambda = n/\sigma^2$. This, together with the first-order stability and Theorem \ref{thm: M contractive} yields the contractive property of the population $M$-function in Theorem \ref{thm: FOS}.

%% file: 3-generalization.tex
Having set up the optimization guarantees for the population EM (Algorithm \ref{alg:pop-EM}) in the previous section, we move to the sample-based setting and establish generalization characteristics the empirical EM. Coupling these two results, we provide convergence guarantees of the (empirical) EM algorithm later in this section. 

Let us recall the empirical setting of our interest where each node $j=1,\cdots,m$ nodes observes $n$ linear measurements denoted by $\{(x_i^j, y_i^j) | i=1,\cdots,n\}$ and generated by the C-MLR model in \eqref{eq: emp-C-MLR model}, that is, $y_i^j = \xi^j \langle x_i^j, \theta^* \rangle + \epsilon_i^j$. In the following, we establish a uniform generalization error bound for the empirical EM update with finitely many nodes $m$ and samples per node $n$.

%%%%
\begin{theorem}[Generalization gap]\label{thm: generalization}
Consider the C-MLR model in \eqref{eq: emp-C-MLR model} with $\SNR \geq 4$, any tolerance probability $\delta \in (0,1)$ and the empirical and population EM operators in \eqref{eq: emp M} and \eqref{eq: pop M} with $mn \geq 192^2 (d + \log(8/\delta))$ and $n - 64 \log m \geq 104(2d + \log(4/\delta))$. Then, with probability at least $1 - \delta$,
%%%%
\begin{align}
    \sup_{\theta \in \mymathbb{Sh}(\varepsilon, r;\theta^*)} \| M_m(\theta) - M(\theta) \|
    \leq
    \sqrt{\| \theta^* \|^2 + \sigma^2} \sqrt{\frac{d + \log(1/\delta)}{mn}}  \cdot \ccalO(1 + \kappa(\varepsilon)).
\end{align}
%%%%
Here, the supermom is over the spherical shell  $\mymathbb{Sh}(\varepsilon, r;\theta^*) \coloneqq \{\theta \in \reals^d : \varepsilon \leq \| \theta - \theta^* \| \leq r\}$ with $r=\| \theta^* \|/14$ and $\kappa(\varepsilon)$ is the contraction factor of the expected EM update characterized in Theorem \ref{thm: FOS}, i.e.,
%%%%
\begin{align}
    \kappa(\varepsilon)
    =
    \big(\| \theta^* \| + \sigma \big)
    \Big( \SNR + \frac{1}{n \varepsilon} \Big) \exp \left( -n \cdot C(\SNR) \right).
\end{align}
%%%%
\end{theorem}
%%%%

%%%%
\begin{proof}
We defer the proof to Appendix \ref{proof: thm generalization}.
\end{proof}
%%%%

Let us provide a useful implication of Theorem \ref{thm: generalization}. Assume the signal-to-noise ratio is a constant larger than $1$ and the total number of samples are at least $mn = \Omega(d + \log(1/\delta))$. Moreover, suppose that the number of nodes is at most $m = \exp(o(n))$, for instance, it grows at a rate polynomial in $n$. Now take the accuracy 
%%%%
\begin{align}
    \varepsilon_\ell =
    \sqrt{\| \theta^* \|^2 + \sigma^2}  \sqrt{\frac{d + \log(1/\delta)}{mn}},
\end{align}
%%%%
which is particularly of our interest in this paper. This pick of the accuracy lower bound yields that for sufficiently large $n$, the expected EM update is contractive, i.e. $\kappa(\varepsilon_\ell) < 1$. Now, we denote by  $\varepsilon^{{\normalfont \text{unif}}}_\ell$ the smallest scalar for which
%%%%
\begin{align}
    \sup_{\theta \in \mymathbb{Sh}(\varepsilon_\ell, \frac{1}{14} \| \theta^* \|;\theta^*)} \| M_m(\theta) - M(\theta) \|
    \leq
    \varepsilon^{{\normalfont \text{unif}}}_\ell
\end{align}
%%%%
with probability at least $1-\delta$. As a result of Theorem \ref{thm: generalization}, we have with high probability that the supermom generalization gap $\| M_m(\theta) - M(\theta) \|$ over the spherical shell $\theta \in \mymathbb{Sh}(\varepsilon_\ell, \| \theta^* \|/14;\theta^*)$ is at most $\varepsilon^{{\normalfont \text{unif}}}_\ell \leq C_{\varepsilon} \varepsilon_\ell$ for a constant $C_{\varepsilon} \geq 1$. To put it differently, for any parameter $\theta$ in a ball around $\theta^*$ with $\| \theta - \theta^* \| \leq \| \theta^* \|/14$, if $\| \theta - \theta^* \| \leq \varepsilon_\ell$, then $\theta$ is already a fairly accurate estimate of $\theta^*$. Otherwise, Theorem \ref{thm: generalization} guarantees that the generalization error of the empirical EM update is with high probability bounded by a constant multiplicative factor of $\varepsilon_\ell$.

%% file: 8-convergence.tex
Having laid out the main two components of our analysis in Theorems \ref{thm: FOS} and \ref{thm: generalization}, we are ready to formally state the main result of the paper.

%%%%
\begin{theorem}[Main] \label{thm: main}
Consider the empirical EM update \eqref{eq: emp M} with $\SNR \geq 4$ and any tolerance probability $\delta \in (0,1)$ and suppose that the initialization $\theta_0$ is in $\mymathbb{B}(r;\theta^*)$ for $r=\| \theta^* \| /14$. Moreover, assume that $mn \geq 192^2 (d + \log(8/\delta))$ and $n \geq 64 \log(m) + 104(2d + \log(4/\delta))$  while $n$ is large enough that $\kappa(\varepsilon_\ell) \leq 1/2$, $\kappa(\varepsilon_\ell) \leq \exp(- C_{\kappa} n)$ for a constant $C_{\kappa}$ and $4 C_{\varepsilon} \varepsilon_\ell \leq r/2$. 
Then, after 
%%%%
\begin{align}
    T
    =
    1
    +
    \frac{1}{2C_{\kappa} n} \log \bigg( mn \cdot \frac{1}{28 C_{\varepsilon}} \cdot \frac{\| \theta^* \|^2}{\| \theta^* \|^2 + \sigma^2} \cdot \frac{1}{d + \log(1/\delta)}\bigg)
\end{align}
%%%%
iterations of Algorithm \ref{alg:emp-EM}, either
%%%%
\vspace{-5pt}
\begin{enumerate}[label=(\roman*)]
    \item $\| \theta_t - \theta^* \| \leq \varepsilon_\ell$ for some iteration $t=0,1,\cdots,T$, or
    \item $\| \theta_T - \theta^* \| \leq 4 C_{\varepsilon} \varepsilon_\ell$ with probability at least $1 - \delta$.
\end{enumerate}
%%%%
\end{theorem}
%%%%

%%%%
\begin{remark}
The result of Theorem \ref{thm: main} implies the following remarks. Let the empirical EM (Algorithm \ref{alg:emp-EM}) be initialized with $\theta_0$ where $\| \theta_0 - \theta^* \| \leq \| \theta^* \|/14$. In addition, consider the C-MLR model in \eqref{eq: emp-C-MLR model} with a constant SNR larger than $4$ where $m$ and $n$ are such that $mn \geq \Omega(d + \log(1/\delta))$ and $n \geq \Omega(\log(m) + d + \log(1/\delta))$, that is, $m$ grows at a rate no greater than $e^{o(n)}$. Then, Theorem \ref{thm: main} implies that for sufficiently large $n$ and after 
%%%%
\begin{align}
    T
    =
    \ccalO(1)
    +
    \frac{1}{n} \log \bigg( \frac{mn}{d} \cdot\frac{\| \theta^* \|^2}{\| \theta^* \|^2 + \sigma^2}\bigg) \cdot \ccalO(1)
    =
    \ccalO(1)
\end{align}
%%%%
iterations, either $\| \theta_t - \theta^* \| \leq \varepsilon_\ell$ for some iteration $t=0,1,\cdots,T$; or otherwise, 
%%%%
\begin{align}
    \| \theta_T - \theta^* \|
    \leq
    \ccalO(\varepsilon_\ell)
    =
    \sqrt{\| \theta^* \|^2 + \sigma^2}  \sqrt{\frac{d + \log(1/\delta)}{mn}} \cdot \ccalO(1),
\end{align}
%%%%
with probability at least $1 - \delta$. Note that since $m \leq e^{o(n)}$, then the iteration complexity is indeed bounded by a constant, that is, $T = \ccalO(1 + 1/n \cdot \log(mn/d)) = \ccalO(1)$.
\end{remark}
%%%%

%%%%
\begin{remark}
We would like to particularly highlight the fact that implications of the above theorem are two-folded. Theorem \ref{thm: main} shows that if the EM method in Algorithm \ref{alg:emp-EM} is applied to the $mn$ samples generated by the C-MLR while honoring the underlying structure (i.e. shared latent variables for samples of any node), after only a constant number of iterations independent of the number of samples, the statistical accuracy $\ccalO(\sqrt{d/(mn)})$ is attained with high probability. On the one hand and regarding the iteration complexity, this is a significant improvement over the benchmark described in Section \ref{sec: benchmark} where the iteration complexity grows logarithmically with the number of samples. On the other hand, Theorem \ref{thm: main} guarantees that  the statistical accuracy  $\ccalO(\sqrt{d/(mn)})$ is indeed achievable by the same EM algorithm.
\end{remark}
%%%%

\subsection{Proof of Theorem \ref{thm: main}} \label{proof: thm main}

\input{7-main-proof}

%% file: 7-main-proof.tex
As mentioned in the theorem's statement, suppose that Algorithm \ref{alg:emp-EM} is initialized with $\theta_0$ such that $\| \theta_0 - \theta^* \| \leq r = \| \theta^* \|/14$ and consider any iteration $t=0,1,\cdots$. We can write that 
%%%%
\begin{align} \label{eq: t+1 to t}
    \| \theta_{t+1} - \theta^* \|
    =
    \| M_m(\theta_t) - \theta^* \| \leq
    \Vert M(\theta_t) - \theta^* \Vert
    +
    \Vert M_m(\theta_t) - M(\theta_t) \Vert.
\end{align}
%%%%
Assume that for all iterates $0 \leq k \leq t$ we have $\| \theta_k - \theta^* \| > \varepsilon_\ell$, otherwise the theorem's first claim is concluded. Then from Theorem \ref{thm: FOS}, for large enough $n$, we have 
%%%%
\begin{align}
    \Vert M(\theta_t) - \theta^* \Vert
    \leq
    \kappa(\varepsilon_\ell) \cdot \Vert \theta_t - \theta^* \Vert,
    \quad \text{for} \quad
    \kappa(\varepsilon_\ell)
    =
    \big(\| \theta^* \| + \sigma \big)
    \Big( \SNR + \frac{1}{n \varepsilon_\ell} \Big) \exp \left( -n \cdot C( \SNR) \right).
\end{align}
%%%%
In particular, note that
%%%%
\begin{align}
    \frac{1}{n \varepsilon_\ell}
    =
    \frac{1}{n}\bigg(\sqrt{\| \theta^* \|^2 + \sigma^2}  \sqrt{\frac{d + \log(1/\delta)}{mn}} \, \bigg)^{-1}
    =
    \ccalO \bigg(\sqrt{\frac{m}{n}}\,\bigg),
\end{align}
%%%%
and since $m$ grows at a rate at most $m = \exp(o(n))$, there exists a constant $C_{\kappa}$ that for large enough $n$, we have $\kappa(\varepsilon_\ell) \leq \exp(- C_{\kappa} n)$ and $\kappa(\varepsilon_\ell) \leq 1/2$. 

In the course of the proof, we show by induction that the iterates remain in the $r$-neighbourhood of $\theta^*$. Assume that  for all iterates $0 \leq k \leq t$ we have $\| \theta_k - \theta^* \| \leq r$ and therefore, $ \Vert M_m(\theta_t) - M(\theta_t) \Vert \leq \varepsilon^{{\normalfont\text{unif}}}_\ell$ with probability at least $1-\delta$. Plugging in \eqref{eq: t+1 to t} we have that with probability at least $1-\delta$
%%%%
\begin{align}  \label{eq: t+1 to t 2}
    \| \theta_{t+1} - \theta^* \|
    \leq
    e^{- C_{\kappa} n} \| \theta_{t} - \theta^* \|
    +
    \varepsilon^{{\normalfont\text{unif}}}_\ell
\end{align}
%%%%
Note that the above inequality also implies that 
$
    \| \theta_{t+1} - \theta^* \|
    \leq
    r/2
    +
    r/2
    =
    r
$, 
where we used the fact that for large enough $n$, we have $\kappa(\varepsilon_\ell) \leq 1/2$. This concludes the induction argument described before, that is for any $t$, if  $\| \theta_k - \theta^* \| > \varepsilon_\ell$ for all $0 \leq k \leq t$, then with probability at least $1-\delta$, we have that $\| \theta_k - \theta^* \| \leq r$ for all $0 \leq k \leq t$. Now, consider the last iterate $T$ and assume that $\| \theta_t - \theta^* \| > \varepsilon_\ell$ for all $0 \leq t \leq T$. We condition the rest of the analysis on the event $\{\Vert M_m(\theta_t) - M(\theta_t) \Vert \leq \varepsilon^{{\normalfont\text{unif}}}_\ell \text{ for all } t=0,\cdots,T-1\}$ which happens with probability at least $1-\delta$. Repeating the argument yielding to \eqref{eq: t+1 to t 2} implies that
%%%%
\begin{align} 
    \| \theta_{T} - \theta^* \|
    \leq
    e^{- C_{\kappa} n T} \| \theta_0 - \theta^* \|
    +
    \sum_{t=0}^{T} \Big(\frac{1}{2}\Big)^t
    \varepsilon^{{\normalfont\text{unif}}}_\ell
    \leq
    e^{- C_{\kappa} n T} \frac{\|\theta^* \|}{14}
    +
    2
    C_{\varepsilon} \varepsilon_\ell.
\end{align}
%%%%
Balancing the two terms above yields that after $T$ iterations for
%%%%
\begin{align} 
    T
    =
    \frac{1}{C_{\kappa} n} \log \Big(\frac{\| \theta^* \|}{28 C_{\varepsilon} \varepsilon_\ell}\Big)
    =
    \frac{1}{2C_{\kappa} n} \log \bigg( mn \! \cdot \! \frac{1}{28 C_{\varepsilon}} \! \cdot \! \frac{\| \theta^* \|^2}{\| \theta^* \|^2 + \sigma^2} \! \cdot \! \frac{1}{d + \log(1/\delta)}\bigg),
\end{align}
%%%%
we have with probability at least $1-\delta$ that
%%%%
\begin{align} 
    \| \theta_T - \theta^* \| 
    \leq 
    4 C_{\varepsilon} \varepsilon_\ell
    =
    4  C_{\varepsilon} \sqrt{\| \theta^* \|^2 + \sigma^2}  \sqrt{\frac{d + \log(1/\delta)}{mn}}
    =
    \sqrt{\| \theta^* \|^2 + \sigma^2}  \sqrt{\frac{d + \log(1/\delta)}{mn}} \cdot \ccalO(1).
\end{align}
%%%%
Note that Algorithm \ref{alg:emp-EM} has to iterate at least for one iteration and since $m = e^{o(n)}$, therefore we can write that $T=\ccalO(1 + 1/n \cdot \log(mn/d)) = \ccalO(1)$.

%% file: 9-conclusion.tex
Data heterogeneity is a major challenge in scaling up distributed learning frameworks such as federated learning. However, there exist underlying structures in the data generation model of such paradigms that can be employed. In this paper, we focus on a particular model of two-component  mixture of linear regressions where $m$ batches of samples each containing $n$ samples with identical latent variable are available. Expectation-Maximization is a popular method to estimate parameters of models with latent variables, while its theoretical analysis is typically complicated. We provide optimization and generalization guarantees for EM algorithm on clustered samples which enables us to characterize its iteration complexity to estimate he true parameters. An interesting follow-up of our work is to implement the EM algorithm in a distributed fashion which is aligned with modern applications such as federated learning. While new challenges such as consensus of local estimates arise, we believe that our techniques and analysis in this paper will be highly applicable.

%% file: 5-opt-proof.tex
We first show the first-order stability of the population $Q$-function \eqref{eq: pop-Q} and then employ the result of Theorem \ref{thm: M contractive} to conclude the contractive property of the operator $M$. As we will show in the proof of Proposition \ref{prop: population EM}, the gradient of the function $Q(\theta' | \theta)$ (with respect to $\theta'$) is as follows,
%%%%
\begin{align} \label{eq: Q grd 1}
    \gr Q(\theta' | \theta)
    =
    - \frac{n}{\sigma^2} \theta' 
    +
    \E\bigg[ \frac{1}{\sigma^2} \sum_{i=1}^{n} X_i Y_i \tanh\bigg( \frac{1}{\sigma^2} \sum_{i=1}^{n} \langle X_i, \theta\rangle Y_i \bigg) \bigg],
\end{align}
%%%%
where the expectation is over i.i.d. feature vectors $X_i \sim \ccalN(0,I_d)$ and response variables $Y_i = \langle X_i, \theta^* \rangle + \epsilon_i$ with i.i.d. Gaussian noises $\epsilon_i \sim \ccalN(0, \sigma^2)$ for $i \in \{1,\cdots,n\}$. To ease the presentation, we use the following short-hand notation throughout the paper,
%%%%
\begin{align}
    Z
    \coloneqq
    \sum_{i=1}^{n} X_i Y_i.
\end{align}
%%%%
Therefore, we can rewrite the gradient of the $Q$-functions as follows
%%%%
\begin{align} \label{eq: Q grd 2}
    \gr Q(\theta' | \theta)
    &=
    - \frac{n}{\sigma^2} \theta' 
    +
    \E\bigg[ \frac{1}{\sigma^2} \sum_{i=1}^{n} X_i Y_i \tanh\Big( \frac{1}{\sigma^2} \langle Z, \theta\rangle \Big) \bigg] \\
    &=
    - \frac{n}{\sigma^2} \theta' 
    +
    \frac{n}{\sigma^2}  \E\bigg[ X_1 Y_1 \tanh\Big( \frac{1}{\sigma^2} \langle Z, \theta \rangle \Big) \bigg],
\end{align}
%%%%
where we used the fact that the expectation in \eqref{eq: Q grd 1} is symmetric with respect to indices $i \in [n]$. Now, we plug in $\theta, M(\theta)$ and $\theta^*$ in \eqref{eq: Q grd 2} and write
%%%%
\begin{align} \label{eq: FOS 1}
    &\quad
    \norm{ \gr Q(M(\theta)|\theta^*) - \gr Q(M(\theta)|\theta)} \\
    &=
    \frac{n}{\sigma^2} \bigg\| \E \bigg[ X_1 Y_1 \bigg(\tanh \Big( \frac{1}{\sigma^2} \langle Z, \theta \rangle \Big) - \tanh \Big(\frac{1}{\sigma^2} \langle Z, \theta^* \rangle \Big)\bigg) \bigg] \bigg\| \\
    &=
    \frac{n}{\sigma^2} \max_{\beta: \Vert \beta \Vert = 1} \E \bigg[ \langle X_1, \beta \rangle Y_1 \bigg(\tanh \Big( \frac{1}{\sigma^2} \langle Z, \theta \rangle\Big) - \tanh \Big(\frac{1}{\sigma^2} \langle Z, \theta^* \rangle \Big)\bigg)\bigg] \\
    &\leq
    \frac{n}{\sigma^2} \max_{\beta: \Vert \beta \Vert = 1} \bigg| \E \bigg[ \langle X_1, \beta \rangle Y_1 \bigg(\tanh \Big( \frac{1}{\sigma^2} \langle Z, \theta \rangle \Big) - \tanh \Big(\frac{1}{\sigma^2} \langle Z, \theta^* \rangle \Big)\bigg)\bigg] \bigg| \\
    &\leq
    \frac{n}{\sigma^2} \max_{\beta: \Vert \beta \Vert = 1} \sqrt{ \E [ \langle X_1, \beta \rangle^2 Y_1^2]} 
    \sqrt{T_1},
\end{align}
%%%%
 where in the last step above, we used Cauchy–Schwarz inequality and the following short-hand notation,
%%%%
\begin{align}
    T_1
    \coloneqq
    \E \bigg[ \bigg(\tanh \Big( \frac{1}{\sigma^2} \langle Z, \theta \rangle \Big) - \tanh \Big(\frac{1}{\sigma^2} \langle Z, \theta^* \rangle \Big)\bigg)^2 \bigg].
\end{align}
%%%%
In the following, we bound both terms in \eqref{eq: FOS 1}, starting with the first term. According to the regression model $Y_1 = \langle X_1, \theta^* \rangle + \epsilon_1$, we can write for any unit-norm $\beta$ that
%%%%
\begin{align}
    \E [ \langle X_1, \beta \rangle^2 Y_1^2]
    &=
    \E[ \langle X_1, \beta \rangle^2 \langle X_1, \theta^* \rangle^2]
    +
    \E[\langle X_1, \beta \rangle^2 \epsilon_1^2]\\
    & \overset{(a)}{\leq}
    3 \| \beta \|^2 \| \theta^* \|^2 
    +
    \sigma^2 \| \beta \|^2 \\
    &=
    3 \| \theta^* \|^2 + \sigma^2,
\end{align}
%%%%
where in $(a)$ we used Lemma 5 from \cite{balakrishnan2017statistical} which shows that for Gaussian vector $X_1 \sim \ccalN(0,I_d)$ and any two fixed vectors $\beta, \theta$, we have $\E[ \langle X_1, \beta \rangle^2 \langle X_1, \theta \rangle^2] \leq 3 \| \beta \|^2 \| \theta \|^2 $. Therefore,
%%%%
\begin{align}
    \max_{\beta: \Vert \beta \Vert = 1} \sqrt{ \E [ \langle X_1, \beta \rangle^2 Y_1^2]}
    \leq
    \sqrt{3} \| \theta^* \| + \sigma.
\end{align}
%%%%

Next, we upper bound the second terms in \eqref{eq: FOS 1}, that is $T_1$. We begin by defining the following three good events for a given $\theta \in \mathbb{B}(\alpha \| \theta^* \|;\theta^*)$ 
%%%%
\begin{gather}
    \ccalE_1 = 
    \left\{
    \langle Z, \theta \rangle \geq \frac{n}{4}(1-\alpha)\| \theta^* \|^2
    \right\},
    \\
    \ccalE_2 = 
    \left\{
    \langle Z, \theta^* \rangle \geq \frac{n}{4}\| \theta^* \|^2
    \right\},
    \\
    \ccalE_3 = 
    \left\{
    |\langle Z, \theta \rangle - \langle Z, \theta^* \rangle| 
    \leq 
    3n \| \theta^* \| \| \theta - \theta^* \|
    \right\},
\end{gather}
%%%%
and letting $\ccalE$ denote their intersection, that is, $\ccalE \coloneqq \ccalE_1 \cap \ccalE_2 \cap \ccalE_3$. Now, we can write that
%%%%
\begin{align} \label{eq: T1 1}
    T_1
    &=
    \E \bigg[ \bigg(\tanh \Big( \frac{1}{\sigma^2} \langle Z, \theta \rangle \Big) 
    -
    \tanh \Big(\frac{1}{\sigma^2} \langle Z, \theta^* \rangle \Big)\bigg)^2 \bigg] \\
    &\leq
    \E \bigg[ \bigg(\tanh \Big( \frac{1}{\sigma^2} \langle Z, \theta \rangle \Big) 
    -
    \tanh \Big(\frac{1}{\sigma^2} \langle Z, \theta^* \rangle \Big)\bigg)^2  \Big| \, \ccalE \bigg] \\
    &\quad +
    \E \bigg[ \bigg(\tanh \Big( \frac{1}{\sigma^2} \langle Z, \theta \rangle \Big) 
    -
    \tanh \Big(\frac{1}{\sigma^2} \langle Z, \theta^* \rangle \Big)\bigg)^2 \Big| \, \ccalE^c \bigg] \cdot
    \Prob(\ccalE^c).
\end{align}
%%%%
The first term above can be bounded as follows,
%%%%
\begin{align} \label{eq: T1 2}
    &\quad
    \E \bigg[ \bigg(\tanh \Big( \frac{1}{\sigma^2} \langle Z, \theta \rangle \Big) 
    -
    \tanh \Big(\frac{1}{\sigma^2} \langle Z, \theta^* \rangle \Big)\bigg)^2  \Big| \, \ccalE \bigg] \\
    &=
    \E \bigg[
    | \langle Z, \theta \rangle - \langle Z, \theta^* \rangle|^2
    \bigg(\frac{\tanh \left(\frac{1}{\sigma^2} \langle Z, \theta \rangle \right) 
    -
    \tanh \left(\frac{1}{\sigma^2} \langle Z, \theta^* \rangle\right)}{ \langle Z, \theta \rangle - \langle Z, \theta^* \rangle}\bigg)^2 \Big| \, \ccalE \bigg] \\
    &\overset{(a)}{\leq}
    \frac{1}{\sigma^2} \E \bigg[
    |\langle Z, \theta \rangle - \langle Z, \theta^* \rangle|^2
    \bigg(1 - \tanh^2 \Big(\frac{1}{\sigma^2} \min \big\{\langle Z, \theta \rangle, \langle Z, \theta^* \rangle \big\}\Big) \bigg)^2 \Big| \, \ccalE \bigg] \\
    &\leq
    9 n^2 \frac{\| \theta^* \|^2}{\sigma^2} \| \theta - \theta^* \|^2
    \bigg(1 - \tanh^2 \Big(\frac{n}{4}(1-\alpha) \frac{\| \theta^* \|^2}{\sigma^2} \Big)  \bigg)^2 \\
    &\leq
    144 n^2 \SNR^2 \| \theta - \theta^* \|^2
    \exp \left(-n(1-\alpha) \SNR^2 \right),
\end{align}
%%%%
where in $(a)$ we used the following inequality (stated and proved in Lemma \ref{lemma: tanh slope}),
%%%%
\begin{align}
    \frac{\tanh(x_2) - \tanh(x_1)}{x_2 - x_1}
    \leq
    \max\{1-\tanh^2(x_1), 1 - \tanh^2(x_2)\},
    \quad
    \text{for all } x_1, x_2 \geq 0.
\end{align}
%%%%
The second term in the RHS of \eqref{eq: T1 1} can be bounded as follows,
%%%%
\begin{align} \label{eq: T1 3}
    &\quad
    \E \bigg[ \bigg(\tanh \Big( \frac{1}{\sigma^2} \langle Z, \theta \rangle \Big) 
    -
    \tanh \Big(\frac{1}{\sigma^2} \langle Z, \theta^* \rangle \Big)\bigg)^2 \Big| \, \ccalE^c \bigg] \cdot
    \Prob(\ccalE^c) \\
    &\leq
    4\Prob(\ccalE^c) \\
    &\overset{(b)}{\leq}
    8 \exp\bigg( -\frac{n}{32} \Big( \frac{1-\alpha}{1+\alpha}\Big)^2\bigg)
    +
    8 \exp\bigg( -  n \min\bigg\{ \frac{1}{16} \Big( \frac{1 - \alpha}{1 + \alpha}\Big)^2 \SNR^2
    ,
    \frac{1}{8} \Big( \frac{1 - \alpha}{1 + \alpha}\Big) \SNR \bigg\}\bigg) \\
    &\quad +
    8 \exp\Big( -\frac{n}{32}\Big)
    +
    8 \exp\Big( -  n \min\Big\{ \frac{1}{16} \SNR^2
    ,
    \frac{1}{8} \SNR \Big\}\Big) \\
    &\quad +
    8 \exp\Big( -\frac{n}{8} \Big)
    +
    8 \exp\Big( -  n \min\Big\{ \SNR^2
    ,
    \frac{\SNR}{2} \Big\}\Big)\\
    &\leq
    \exp(- 2 c_1 n),
\end{align}
%%%%
for a constant $c_1$ depending on $\alpha$ and $\SNR$. In inequality $(b)$ above, we used the high probability of the good event $\ccalE$ stated and proved in Lemma \ref{lemma: events E1-3}. Putting \eqref{eq: T1 2} and \eqref{eq: T1 3} in \eqref{eq: T1 1} yields that
%%%%
\begin{align}
    \frac{T_1}{\| \theta - \theta^* \|^2}
    &=
    \frac{1}{\| \theta - \theta^* \|^2}
    \E \bigg[ \bigg(\tanh \left( \frac{1}{\sigma^2} \langle Z, \theta \rangle \right) 
    -
    \tanh \Big(\frac{1}{\sigma^2} \langle Z, \theta^* \rangle \Big)\bigg)^2 \bigg] \\
    &\leq
    144  n^2 \SNR^2
    \exp \Big(-n(1-\alpha)\SNR^2 \Big)
    +
    \frac{\exp(- 2 c_1 n)}{\| \theta - \theta^* \|^2} \\
    &\leq
    144 n^2 \SNR^2 \exp \Big(-n(1-\alpha)\SNR^2 \Big)  
    +
    \varepsilon^{-2} \exp(- 2 c_1 n),
\end{align}
%%%%
where we assume that $\| \theta - \theta^* \| \geq \varepsilon$. Therefore,
%%%%
\begin{align}
    \frac{\sqrt{T_1}}{\| \theta - \theta^* \|}
    &\leq
    12 \SNR \cdot n \exp \Big(-\frac{n}{2} (1-\alpha)\SNR^2 \Big)  
    +
    \varepsilon^{-1} \exp(- c_1 n)
\end{align}
%%%%
Putting all together in \eqref{eq: FOS 1}, we have the following FOS satisfied
%%%%
\begin{align}
    \Vert \gr Q(M(\theta)|\theta^*) - \gr Q(M(\theta)|\theta) \Vert
    \leq
    \gamma \Vert \theta - \theta^* \Vert,
\end{align}
%%%%
for every $\theta$ in $\mymathbb{B}(\alpha \| \theta^* \|;\theta^*)$ such that $\| \theta - \theta^* \| \geq \varepsilon$. Here, the $\gamma$ parameter is
%%%%
\begin{align}
    \gamma
    &=
    \frac{1}{\sigma^2} \bigg( 12 n \frac{\| \theta^* \|}{\sigma} \exp \Big(-\frac{n}{2} (1-\alpha)\SNR^2\Big)  
    +
    \varepsilon^{-1} \exp(- c_1 n) \bigg)
    \cdot
    ( \sqrt{3} \| \theta^* \| + \sigma ) \\
    &\leq
    \frac{1}{\sigma^2} (\| \theta^* \| + \sigma )
    \Big( n \cdot \SNR + \frac{1}{\varepsilon} \Big) \exp ( - n \cdot C(\alpha, \SNR) ),
\end{align}
%%%%
and for $n \geq N_0(\alpha, \SNR)$ where both $N_0(\alpha, \SNR)$ and $C(\alpha, \SNR)$ are constants depending on $c_1$, $\alpha$ and $\SNR$ (and therefore depending on $\alpha$ and $\SNR$). Moreover, $Q(\cdot | \theta^*)$ is $\lambda$--strongly concave with $\lambda = n/\sigma^2$. Following the proof of Theorem 1 in \cite{balakrishnan2017statistical}, it can be shown that for every $\theta$ in $\mathbb{B}(\alpha \| \theta^* \|;\theta^*)$ with $\| \theta - \theta^* \| \geq \varepsilon$, we have
%%%%
\begin{align}
    \Vert M(\theta) - \theta^* \Vert
    \leq
    \frac{\gamma}{\lambda} \Vert \theta - \theta^* \Vert,
\end{align}
%%%%
where in our case
%%%%
\begin{align}
    \frac{\gamma}{\lambda}
    \leq
    \kappa
    \coloneqq
    ( \| \theta^* \| + \sigma)
    \Big( \SNR + \frac{1}{n \varepsilon} \Big) \exp ( - n \cdot C(\alpha, \SNR) ),
\end{align}
%%%%
which concludes Theorem \ref{thm: FOS}'s claim. It is worth noting that the constraint  $\| \theta - \theta^* \| \geq \varepsilon$ in our case does not affect the conclusion of Theorem 1 in \cite{balakrishnan2017statistical}.

%%%%%%%%%%%%%%%%%%%%%%%%%%%%%%%%%%%%%%%%%%%%%%%%%%%%
\subsection{Useful lemmas and proofs}
%%%%%%%%%%%%%%%%%%%%%%%%%%%%%%%%%%%%%%%%%%%%%%%%%%%%

%%%%
\begin{lemma} \label{lemma: events E1-3}
Assume that $\|\theta - \theta^* \| \leq \alpha \| \theta^* \|$ for some $0 \leq \alpha < 1$ and let $\SNR = \| \theta^* \| / \sigma$ denote the SNR. Then, the following three events are high probability,
%%%%
\begin{gather}
    \ccalE_1 = 
    \left\{
    \langle Z, \theta \rangle \geq \frac{n}{4}(1-\alpha)\| \theta^* \|^2
    \right\},
    \\
    \ccalE_2 = 
    \left\{
    \langle Z, \theta^* \rangle \geq \frac{n}{4}\| \theta^* \|^2
    \right\},
    \\
    \ccalE_3 = 
    \left\{
    |\langle Z, \theta \rangle - \langle Z, \theta^* \rangle| 
    \leq 
    3n \| \theta^* \| \| \theta - \theta^* \|
    \right\}.
\end{gather}
%%%%
In particular, we have that
%%%%
\begin{gather}
    \Prob(\ccalE_1)
    \geq 
    1 - 
    2 \exp\bigg( \! -\frac{n}{32} \Big( \frac{1-\alpha}{1+\alpha}\Big)^2\bigg)
    -
    2 \exp\bigg( \! -  n \min\bigg\{ \frac{1}{16} \Big( \frac{1 - \alpha}{1 + \alpha}\Big)^2 \SNR^2
    ,
    \frac{1}{8} \Big( \frac{1 - \alpha}{1 + \alpha}\Big) \SNR \bigg\}\bigg), \\
    \Prob(\ccalE_2)
    \geq 
    1 - 
    2 \exp\Big(  \!-\frac{n}{32}\Big)
    -
    2 \exp\bigg(  \!-  n \min\bigg\{ \frac{1}{16} \SNR^2
    ,
    \frac{1}{8} \SNR \bigg\}\bigg),\\
    \Prob(\ccalE_3)
    \geq 
    1 - 
    2 \exp\Big( \! -\frac{n}{8} \Big)
    -
    2 \exp\bigg(  \!-  n \min\bigg\{ \SNR^2
    ,
    \frac{\SNR}{2} \bigg\}\bigg).
\end{gather}
%%%%
\end{lemma}
%%%%

%%%%%%%%%%%%%%%%%%%%%%%%%%%%%%%%%%%%%%%%%%%%%%%%%%%%
\subsection{Proof of Lemma \ref{lemma: events E1-3}}
%%%%%%%%%%%%%%%%%%%%%%%%%%%%%%%%%%%%%%%%%%%%%%%%%%%%
To prove the first two parts of the lemma, we employ the following result on the concentration of sub-exponential random variables.
%%%%
\begin{lemma} \label{lemma: signal O(n) bound}
Consider the linear regression model $Y_i = \langle \theta^*, X_i \rangle + \epsilon_i$ where $X_i \sim \ccalN(0,I)$ and $\epsilon \sim \ccalN(0,\sigma^2)$ are independent and $i \in \{1,\cdots,n\}$. Also, assume that $\| \theta - \theta^*\| \leq \alpha \| \theta^* \|$ for some $0 \leq \alpha < 1$. Then, for any $n$,
%%%%
\begin{align}
    \bigg| \sum_{i=1}^{n} \langle X_i, \theta\rangle \langle X_i, \theta^*\rangle - n \langle \theta, \theta^*\rangle \bigg|
    \leq
    \frac{n}{2} \langle \theta, \theta^*\rangle,
    \text{ w.p. at least }
    1 -  2 \exp\bigg( -\frac{n}{32} \Big( \frac{1-\alpha}{1+\alpha}\Big)^2 \bigg),
\end{align}
%%%%
and
%%%%
\begin{align}
    \bigg|\sum_{i=1}^{n} \langle X_i, \theta\rangle \epsilon_i \bigg|
    \leq
    \frac{n}{4}\langle \theta, \theta^*\rangle,
    \text{ w.p. at least }
    1 - 2 \exp\bigg( \! \! -  n \min\bigg\{ \frac{1}{16} \Big( \frac{1 - \alpha}{1 + \alpha}\Big)^2 \SNR^2
    ,
    \frac{1}{8} \Big( \frac{1 - \alpha}{1 + \alpha}\Big) \SNR \bigg\}\bigg),
\end{align}
%%%%
where $\SNR \coloneqq \| \theta^* \| / \sigma$.
\end{lemma}
%%%%

%%%%%%%%%%%%%%%%%%%%%%%%%%%%%%%%%%%%%%%%%%%%%%%%%%%%
\subsection{Proof of Lemma \ref{lemma: signal O(n) bound}}
%%%%%%%%%%%%%%%%%%%%%%%%%%%%%%%%%%%%%%%%%%%%%%%%%%%%

Let us denote signal random variables $S_i \coloneqq \langle X_i, \theta\rangle \langle X_i, \theta^*\rangle$ where $\E[S_i] = \langle \theta, \theta^*\rangle \geq (1 - \alpha) \| \theta^* \|^2 > 0$ for each $i \in \{1,\cdots,n\}$. As shown in Lemma \ref{lemma: sub-exponential}, $S_i$s are i.i.d. sub-exponential, in particular, $S_i \sim \SubE(4 \| \theta \|^2 \| \theta^* \|^2, 4 \| \theta \| \| \theta^* \|)$. This yields that 
%%%%
\begin{align} 
    \sum_{i=1}^{n} \langle X_i, \theta\rangle \langle X_i, \theta^*\rangle
    =
    \sum_{i=1}^{n} S_i
    \sim
    \SubE\left(4n \| \theta \|^2 \| \theta^* \|^2, 4\| \theta \| \| \theta^* \| \right),
\end{align}
%%%%
and therefore for any $t \geq 0$, we have the following concentration of sum of $S_i$s around its mean value, that is,
%%%%
\begin{align}
    \Prob\bigg( \bigg| \sum_{i=1}^{n} S_i - n \langle \theta, \theta^*\rangle \bigg| \geq nt \bigg)
    \leq
    2 \exp\bigg( -  \min\bigg\{ \frac{nt^2}{8 \| \theta \|^2 \| \theta^* \|^2}, \frac{nt}{8 \| \theta \| \| \theta^* \|} \bigg\}\bigg).
\end{align}
%%%%
We pick $t = \langle \theta, \theta^*\rangle/2$ which yields that
%%%%
\begin{align}
    \min\bigg\{ \frac{nt^2}{8 \| \theta \|^2 \| \theta^* \|^2}, \frac{nt}{8 \| \theta \| \| \theta^* \|} \bigg\}
    &=
    \min\bigg\{ \frac{n}{32} \bigg( \frac{\langle \theta, \theta^*\rangle}{\| \theta \| \| \theta^* \|}\bigg)^2
    ,
    \frac{n}{16} \bigg( \frac{\langle \theta, \theta^*\rangle}{\| \theta \| \| \theta^* \|}\bigg)\bigg\} \\
    &=
    \frac{n}{32} \bigg( \frac{\langle \theta, \theta^*\rangle}{\| \theta \| \| \theta^* \|}\bigg)^2\\
    &\geq
    \frac{n}{32} \Big( \frac{1-\alpha}{1+\alpha}\Big)^2.
\end{align}
%%%%
Therefore, for any $n$ we have that
%%%%
\begin{align} \label{eq: sumS}
    \Prob\bigg( \bigg| \sum_{i=1}^{n} S_i - n \langle \theta, \theta^*\rangle \bigg| \leq \frac{n}{2} \langle \theta, \theta^*\rangle \bigg)
    &\geq
    1 - 2 \exp\bigg( -\frac{n}{32} \left( \frac{1-\alpha}{1+\alpha}\right)^2\bigg).
\end{align}
%%%%
Next, we define i.i.d. noise signals $N_i \coloneqq \langle X_i, \theta\rangle \epsilon_i$ for $i \in \{1,\cdots,n\}$ where we have $\E[N_i]=0$. As we show in Lemma \ref{lemma: sub-exponential},  
$N_i$s are sub-exponential random variables with parameters $N_i \sim \SubE(\| \theta \|^2 \sigma^2/2, \| \theta \| \sigma)$. Now, we can write concentration for sum of $N_i$s as follows. We have that
%%%%
\begin{align} 
    \sum_{i=1}^{n} \langle X_i, \theta\rangle \epsilon_i
    =
    \sum_{i=1}^{n} N_i
    \sim
    \SubE\bigg(\frac{1}{2}n\| \theta \|^2 \sigma^2, \| \theta \| \sigma \bigg) \label{eq: noise SubE}
\end{align}
%%%%
and therefore for any $t \geq 0$, we have
%%%%
\begin{align}
    \Prob\bigg( \bigg| \sum_{i=1}^{n} N_i \bigg| \geq nt \bigg)
    \leq
    2 \exp\bigg( -  \min\bigg\{ \frac{nt^2}{\| \theta \|^2 \sigma^2}, \frac{nt}{2\| \theta \| \sigma} \bigg\}\bigg).
\end{align}
%%%%
In particular, it yields for $t = \langle \theta, \theta^*\rangle/4$ that
%%%%
\begin{align}
    \min\bigg\{ \frac{nt^2}{\| \theta \|^2 \sigma^2}, \frac{nt}{2\| \theta \| \sigma} \bigg\}
    &=
    \min\bigg\{ \frac{n}{16} \bigg( \frac{\langle \theta, \theta^*\rangle}{\| \theta \| \sigma}\bigg)^2
    ,
    \frac{n}{8} \bigg( \frac{\langle \theta, \theta^*\rangle}{\| \theta \| \sigma}\bigg) \bigg\} \\
    &\geq
    n \min\bigg\{ \frac{1}{16} \Big( \frac{1 - \alpha}{1 + \alpha}\Big)^2 \SNR^2
    ,
    \frac{1}{8} \Big( \frac{1 - \alpha}{1 + \alpha}\Big) \SNR \bigg\},
\end{align}
%%%%
and consequently,
%%%%
\begin{align} \label{eq: sumN}
    \Prob\bigg( \bigg| \sum_{i=1}^{n} N_i \bigg|
    \leq \frac{n}{4}\langle \theta, \theta^*\rangle \bigg)
    \geq
    1 - 2 \exp\bigg(\! -  n \min\bigg\{ \frac{1}{16} \Big( \frac{1 - \alpha}{1 + \alpha}\Big)^2 \SNR^2
    ,
    \frac{1}{8} \Big( \frac{1 - \alpha}{1 + \alpha}\Big) \SNR \bigg\}\bigg).
\end{align}
%%%%
Putting the two high probability events in \eqref{eq: sumS} and \eqref{eq: sumN} implies that event $\ccalE_1$ holds, that is,
%%%%
\begin{align} 
    \langle Z, \theta \rangle
    &=
    \sum_{i=1}^{n} \langle X_i, \theta \rangle Y_i\\
    &=
    \sum_{i=1}^{n} \langle X_i, \theta\rangle \langle X_i, \theta^*\rangle 
    +
    \sum_{i=1}^{n} \langle X_i, \theta\rangle \epsilon_i \\
    &\geq
    \frac{n}{4}\langle \theta, \theta^*\rangle \\
    &\geq
    \frac{n}{4}(1-\alpha)\| \theta^* \|^2,
\end{align}
%%%%
with probability $\Prob(\ccalE_1)$ stated in the lemma. As a particular case, we set $\theta = \theta^*$ (and thus $\alpha=0$) in above and conclude the high probability of event $\ccalE_2$. Next, we move to show the high probability of event $\ccalE_3$.
We have that
%%%%
\begin{align} 
    \langle Z, \theta \rangle - \langle Z, \theta^* \rangle
    =
    \sum_{i=1}^{n} \langle X_i, \theta - \theta^* \rangle Y_i
    =
    \sum_{i=1}^{n} \langle X_i, \theta - \theta^*\rangle \langle X_i, \theta^*\rangle 
    +
    \sum_{i=1}^{n} \langle X_i, \theta - \theta^*\rangle \epsilon_i.
\end{align}
%%%%
From Lemma \ref{lemma: sub-exponential}, we have
%%%%
\begin{align} 
    \sum_{i=1}^{n} \langle X_i, \theta - \theta^*\rangle \langle X_i, \theta^*\rangle 
    \sim
    \SubE(4 \| \theta^* \|^2 \| \theta - \theta^* \|^2, 4 \| \theta^* \| \| \theta - \theta^* \|).
\end{align}
%%%%
Therefore, with probability at least $1 -  2 \exp(-n/8)$,
%%%%
\begin{gather}
    \bigg| \sum_{i=1}^{n} \langle X_i, \theta - \theta^*\rangle \langle X_i, \theta^*\rangle - n \langle  \theta - \theta^*, \theta^*\rangle \bigg|
    \leq
    n \| \theta^* \| \| \theta - \theta^* \|,
\end{gather}
%%%%
implying that
%%%%
\begin{gather}
    \bigg| \sum_{i=1}^{n} \langle X_i, \theta - \theta^*\rangle \langle X_i, \theta^*\rangle \bigg|
    \leq
    n |\langle  \theta - \theta^*, \theta^*\rangle|
    +
    n \| \theta^* \| \| \theta - \theta^* \|
    \leq
    2n \| \theta^* \| \| \theta - \theta^* \|.
\end{gather}
%%%%
On the other hand,
%%%%
\begin{align} 
    \sum_{i=1}^{n} \langle X_i, \theta - \theta^*\rangle \epsilon_i
    \sim
    \SubE\left(\frac{1}{2}n\| \theta - \theta^* \|^2 \sigma^2, \| \theta - \theta^* \| \sigma \right)
\end{align}
%%%%
Therefore,
%%%%
\begin{gather}
    \bigg|\sum_{i=1}^{n} \langle X_i, \theta - \theta^*\rangle \epsilon_i \bigg|
    \leq
    n \| \theta - \theta^* \| \| \theta^* \|,
\end{gather}
%%%%
with probability at least
%%%%
\begin{align}
    1 - 2 \exp\bigg( -  n \min\left\{ \SNR^2
    ,
    \frac{\SNR}{2} \right\}\bigg).
\end{align}
%%%%
Therefore,
%%%%
\begin{align} 
    \Prob(\ccalE_3) 
    &=
    \Prob \left(\left\{
    |\langle Z, \theta \rangle - \langle Z, \theta^* \rangle| 
    \leq 
    3n \| \theta^* \| \| \theta - \theta^* \|
    \right\} \right) \\
    &\geq
    1 - 
    2 \exp\Big( -\frac{n}{8} \Big)
    -
    2 \exp\bigg( -  n \min\left\{ \SNR^2
    ,
    \frac{\SNR}{2} \right\}\bigg).
\end{align}
%%%%

%%%%
\begin{lemma} \label{lemma: sub-exponential}
Let $X \sim \ccalN(0, I_d)$ be Gaussian. Then, for any $u,v \in \reals^d$, $\langle X, u \rangle \langle X, v \rangle$ is sub-exponential $\SubE(\tau^2,b)$ with
%%%%
\begin{align}
    \tau^2
    =
    4 \| u \|^2 \| v \|^2,
    \quad
    \text{and}
    \quad
    b
    =
    4 \| u \| \| v \|.
\end{align}
%%%%
Also, assume that $\epsilon \sim \ccalN(0,\sigma^2)$ is independent of $X$. Then, $\langle X, u \rangle \epsilon \sim \SubE(\| u \|^2 \sigma^2/2, \| u \| \sigma)$.
\end{lemma}
%%%%

\subsection{Proof of Lemma \ref{lemma: sub-exponential}}

Define $S = \langle X, u \rangle \langle X, v \rangle$. We can write
%%%%
\begin{align}
    \frac{S}{\| u \| \| v \|} 
    =
    \langle X, u/\| u \| \rangle \langle X, v/\| v \| \rangle 
    =
    \langle X, \ubar \rangle \langle X, \vbar \rangle
    = \frac{1}{4} \langle X, \ubar + \vbar \rangle^2 
    -
    \frac{1}{4} \langle X, \ubar - \vbar \rangle^2,
\end{align}
%%%%
where we denote $\ubar \coloneqq u / \| u \|$ and $\vbar \coloneqq v / \| v \|$.
Moreover, $S_1 \coloneqq \langle X, \ubar + \vbar \rangle$ and $S_2 \coloneqq \langle X, \ubar - \vbar \rangle$ are zero-mean Gaussian RVs and
%%%%
\begin{align}
    \E[S_1 S_2]
    =
    \E[(\ubar + \vbar)^\top X X^\top (\ubar - \vbar)]
    =
    \langle \ubar + \vbar, \ubar - \vbar \rangle
    =
    0,
\end{align}
%%%%
implying that $S_1$ and $S_2$ are independent. Therefore, we can write that
%%%%
\begin{align} \label{eq: S subE}
    \E\left[ \exp\left(\lambda  \left(S - \mu_S\right) \right) \right]
    &=
    \E \bigg[ \exp\bigg(\lambda \| u \| \| v \| \bigg(\frac{S}{\| u \| \| v \|} - \frac{\mu_S}{\| u \| \| v \|}\bigg) \bigg) \bigg] \\
    &=
    \E \bigg[ \exp\bigg( \frac{\lambda}{4} \| u \| \| v \| \left(S_1^2 - \mu_{S_1^2}\right) \bigg) \bigg] \\
    &\quad \cdot
    \E \bigg[ \exp\bigg( -\frac{\lambda}{4} \| u \| \| v \| \left(S_2^2 - \mu_{S_2^2}\right) \bigg) \bigg], 
\end{align}
%%%%
where we used the independence of $S_1$ and $S_2$ in above. Next, we employ sub-exponential property of $S_1^2$ and $S_2^2$. More precisely, $S_1^2 \sim \SubE(4 \sigma_1^4, 4 \sigma_1^2)$ and $S_1^2 \sim \SubE(4 \sigma_2^4, 4 \sigma_2^2)$ where
%%%%
\begin{align}
    \sigma_1^2 \coloneqq \E[S_1^2] = 2 \left(1 + \langle \ubar , \vbar \rangle \right),
    \quad
    \text{and}
    \quad
    \sigma_2^2 \coloneqq \E[S_2^2] = 2 \left(1 - \langle \ubar , \vbar \rangle \right).
\end{align}
%%%%
In other words, 
%%%%
\begin{align}
    \E\bigg[ \exp\bigg(\lambda'  \left(S_1^2 - \mu_{S_1^2}\right) \bigg) \bigg]
    &\leq
    \exp\left(2 {\lambda'}^2 \sigma_1^4\right), 
    \quad
    \forall |\lambda'| \leq \frac{1}{4 \sigma_1^2},\\
    \E\bigg[ \exp\bigg(\lambda'  \left(S_2^2 - \mu_{S_2^2}\right) \bigg) \bigg]
    &\leq
    \exp\left(2 {\lambda'}^2 \sigma_2^4\right), 
    \quad
    \forall |\lambda'| \leq \frac{1}{4 \sigma_2^2}.
\end{align}
%%%%
Putting $\lambda' = \pm \frac{\lambda}{4} \| u \| \| v \|$ and using \eqref{eq: S subE}, we have that
%%%%
\begin{align} 
    \E\bigg[ \exp\bigg(\lambda  \left(S - \mu_S\right) \bigg) \bigg]
    &\leq
    \exp\bigg(\frac{\lambda^2}{8} \| u \|^2 \| v \|^2\left( \sigma_1^4 + \sigma_2^4\right) \bigg),
    \quad
    \forall |\lambda| \leq \frac{1}{ \| u \| \| v \| \max\{\sigma_1^2, \sigma_2^2\}}.
\end{align}
%%%%
Finally,
%%%%
\begin{align} 
    \sigma_1^4 + \sigma_2^4
    =
    4  \left(1 + \langle \ubar , \vbar \rangle \right)^2
    +
    4 \left(1 - \langle \ubar , \vbar \rangle \right)^2
    =
    8 \left(1 + \langle \ubar , \vbar \rangle^2\right)
    \leq
    16,
\end{align}
%%%%
yielding that
%%%%
\begin{align} 
    \| u \|^2 \| v \|^2 \left( \sigma_1^4 + \sigma_2^4 \right)
    \leq
    16 \| u \|^2 \| v \|^2,
\end{align}
%%%%
and
%%%%
\begin{align} 
    \| u \| \| v \| \max\{\sigma_1^2, \sigma_2^2\}
    =
    2 \| u \| \| v \| \max\{ 1 + \langle \ubar , \vbar \rangle, 1 - \langle \ubar , \vbar \rangle \}
    =
    2 \| u \| \| v \| (1 +  | \langle \ubar , \vbar \rangle|)
    \!\leq \!
    4 \| u \| \| v \|.
\end{align}
%%%%

Next, consider $N = \langle X, u \rangle \epsilon$. We know that $N$ is sub-exponential with parameters $\SubE(\tau_N^2, b_N)$. Denote $Z_x \coloneqq \frac{\langle X, u\rangle}{\| u \|}$ and $Z_{\epsilon} \coloneqq \frac{\epsilon}{\sigma}$. Clearly, $Z_x$ and $Z_{\epsilon}$ are independent standard Gaussian. Moreover,
%%%%
\begin{align} 
    N
    =
    \| u \| \sigma Z_x Z_{\epsilon}
    =
    \frac{\| u \| \sigma}{4} \left( Z_x + Z_{\epsilon}\right)^2
    -
    \frac{\| u \| \sigma}{4} \left( Z_x - Z_{\epsilon}\right)^2.
\end{align}
%%%%
Furthermore, $Z_x + Z_{\epsilon}$ and $Z_x - Z_{\epsilon}$ are zero-mean Gaussian RVs and $\E[(Z_x + Z_{\epsilon})(Z_x - Z_{\epsilon})] = 0$, implying that $Z_x + Z_{\epsilon}$ and $Z_x - Z_{\epsilon}$ are independent. Therefore,
%%%%
\begin{align} 
    \E\left[ \exp\left(\lambda  N \right) \right] 
    &=
    \E\left[ \exp\left(\lambda \| u \| \sigma Z_x Z_{\epsilon} \right) \right]\\
    &=
    \E\bigg[ \exp\bigg(\frac{\lambda \| u \| \sigma}{4} \left( Z_x + Z_{\epsilon}\right)^2
    -
    \frac{\lambda \| u \| \sigma}{4} \left( Z_x - Z_{\epsilon}\right)^2 \bigg) \bigg]\\
    &=
    \E\bigg[ \exp\bigg(\frac{\lambda \| u \| \sigma}{4} \left( Z_x + Z_{\epsilon}\right)^2 \bigg) \bigg]
    \cdot
    \E\bigg[ \exp\bigg(
    - \frac{\lambda \| u \| \sigma}{4} \left( Z_x - Z_{\epsilon}\right)^2 \bigg) \bigg]\\
    &\leq
    \exp\bigg( \frac{\lambda^2}{2} \frac{\| u \|^2 \sigma^2}{2} \bigg),
\end{align}
%%%%
for any $|\lambda| \leq \frac{1}{\| u \| \sigma}$. This concludes that $N \sim \SubE(\| u \|^2 \sigma^2/2, \| u \| \sigma)$. In above, we use the fact that if $Z \sim \ccalN(0,1)$, then $Z^2 \sim \SubE(4,4)$.

%% file: 4-gen-proof.tex
We begin by setting up a few shorthand notations as follows
%%%%
\begin{gather} 
    \SigmaHat
    =
    \frac{1}{mn} \sum_{j=1}^{m} \sum_{i=1}^{n} x_i^j {x_i^j}^{\top},
    \quad\quad
    Z^j
    =
    \sum_{i=1}^{n} x_i^j y_i^j, \\
    \vHat
    =
    \frac{1}{mn} \sum_{j=1}^{m} \sum_{i=1}^{n} x_i^j y_i^j \tanh\bigg( \frac{1}{\sigma^2} \sum_{i=1}^{n} \langle x_i^j, \theta\rangle y_i^j \bigg)
    =
    \frac{1}{mn} \sum_{j=1}^{m} Z^j \tanh\bigg( \frac{1}{\sigma^2} \langle Z^j, \theta\rangle  \bigg), \\
    v
    =
    \E \bigg[X_1 Y_1 \tanh\bigg( \frac{1}{\sigma^2} \sum_{i=1}^{n} \langle X_i, \theta\rangle Y_i \bigg) \bigg].\label{eq: notation 1}
\end{gather}
%%%%
Therefore, we can write that
%%%%
\begin{align} \label{eq: Mm - M}
    \| M_m(\theta) - M(\theta) \|
    =
    \| \SigmaHat^{-1} \vHat - v \|
    \leq
    \underbrace{\| \SigmaHat^{-1} \|_{\op}  \| v - \vHat \|}_{T_1}
    +
    \underbrace{\| \SigmaHat^{-1} - I \|_{\op}  \| v \|}_{T_2}.
\end{align}
%%%%
In the following, we bound each of the two terms above.

%%%%%%%%%%%%%%%%%%%%%%%%%%%%%%%%%%%%%%%%%%%%%%%%%%%%%%%%%%%%%%%%
\subsubsection{Bounding $T_1$:}
%%%%%%%%%%%%%%%%%%%%%%%%%%%%%%%%%%%%%%%%%%%%%%%%%%%%%%%%%%%%%%%%
We use the following concentration bounds in bounding both $T_1$ and $T_2$.
%%%%
\begin{lemma} \label{lemma: gaussian covariance}
Let standard Gaussian random variables $X_i \sim \ccalN(0,I_d)$ be independent for $i=1,\cdots,N$. For any $\delta \in (0,1)$, if $N \geq 192^2 (d + \log(2/\delta))$, then we have that
%%%%
\begin{align}
    \big\Vert \SigmaHat_N - I_d \big\Vert_{\op}
    \leq
    96 \sqrt{\frac{d + \log(2/\delta)}{N}},
    \quad
    \big\Vert \SigmaHat_N^{-1} - I_d \big\Vert_{\op}
    \leq
    192 \sqrt{\frac{d + \log(2/\delta)}{N}},
    \quad
    \big\Vert \SigmaHat_N^{-1} \big\Vert_{\op}
    \leq
    2,
\end{align}
%%%%
each with probability at least $1 - \delta$. Here, we denote the sample covariance matrix of the $N$ samples by
%%%%
\begin{align}
    \SigmaHat_N
    \coloneqq
    \frac{1}{N} \sum_{i=1}^{N} x_i x_i^\top.
\end{align}
%%%%
\end{lemma}
%%%%
As a result of Lemma \ref{lemma: gaussian covariance}, if $mn \geq 192^2 (d + \log(2/\delta))$, then $\| \SigmaHat^{-1} \|_{\op} \leq 2$ and therefore $T_1 \leq 2  \| v - \vHat \|$ with probability $1-\delta$. Next, we upper bound $\| v - \vHat \|$ by first decomposing it to the following three terms
%%%%
\begin{gather}
    \| \vHat - v \|
    \leq
    \| \vHat -  \vHat_0 \|
    +
    \| v -  v_0 \|
    +
    \| \vHat_0  -  v_0 \|, \label{eq: vhat - v decomp}
\end{gather}
%%%%
where we use the following notations
%%%%
\begin{gather}
    \vHat_0
    =
    \frac{1}{mn} \sum_{j=1}^{m} Z^j,
    \quad
    v_0
    =
    \frac{1}{n} \E \left[ Z \right],
    \quad
    Z
    =
    \sum_{i=1}^{n} X_i Y_i. \label{eq: notation 2}
\end{gather}
%%%%
Here, since $Z^j$s are i.i.d. across different nodes (i.e. different $j$s), we denote by $Z$ the generic random variable with the same distribution as $Z^j$ for any $j=1,\cdots,m$. It is also worth noting that $v_0$ and $\vHat_0$ defined above approach $v$ and $\vHat$ defined in \eqref{eq: notation 1} respectively, as the $\tanh(\cdot)$ terms therein approach $1$. In the following three lemmas, we upper bound each of the terms in \eqref{eq: vhat - v decomp}.
%%%%
\begin{lemma} \label{lemma: v - v0}
Assuming that $n \geq d$, there exist $C_4(\alpha,\SNR)$ and $N_1(\alpha, \SNR)$, constants depending on $0 \leq \alpha < 1$ and $\SNR$, such that for any $\theta \in \mymathbb{B}(\alpha \| \theta^* \|; \theta^*)$ and $n \geq N_1(\alpha,\SNR)$, we have that
%%%%
\begin{align}
    \| v -  v_0 \|
    &\leq
    (1 + \| \theta^* \| + \sigma) \exp \left(- n \cdot C_4(\alpha, \SNR ) \right).
\end{align}
%%%%
\end{lemma}
%%%%
%%%%
\begin{proof}
We defer the proof to Section \ref{proof: lemma v - v0}.
\end{proof}
%%%%

%%%%
\begin{lemma} \label{lemma: vHat0 - v0}
Fix $\delta \in (0,1)$ and assume that $mn \geq 32^2 (2d + \log(1/\delta))$. Then, with probability at least $1-\delta$, we have 
%%%%
\begin{align}
    \| \vHat_0 -  v_0 \|
    &\leq
    8\sqrt{\| \theta^* \|^2 + \sigma^2} \sqrt{\frac{2d + \log(1/\delta)}{mn}}.
\end{align}
%%%%
\end{lemma}
%%%%

%%%%
\begin{proof}
We defer the proof to Section \ref{proof: lemma vHat0 - v0}.
\end{proof}
%%%%

%%%%
\begin{lemma} \label{lemma: vHat - vHat0}
For $r = \frac{1}{14} \| \theta^* \|$ and $\SNR \geq 4$, with probability at least $1 - (m+2) \cdot 5^d \cdot \exp(-n/64)$,
%%%%
\begin{align}
    \sup_{\theta \in \mymathbb{B}(r;\theta^*)} \| \vHat -  \vHat_0 \|
    &\leq
    2 ( 3 \| \theta^* \| + \sigma) \exp(- 4 n).
\end{align}
%%%%
\end{lemma}
%%%%

%%%%
\begin{proof}
We defer the proof to Section \ref{proof: lemma vHat - vHat0}.
\end{proof}
%%%%

Putting the results of the above three lemmas back in the decomposition of $\| v - \vHat \|$ in \eqref{eq: vhat - v decomp} yields that
%%%%
\begin{align} \label{eq: sup vHat - v}
    \sup_{\theta \in \mymathbb{B}(r;\theta^*)} \| \vHat - v \|
    &\leq
    \sup_{\theta \in \mymathbb{B}(r;\theta^*)} \| \vHat -  \vHat_0 \|
    +
    \sup_{\theta \in \mymathbb{B}(r;\theta^*)} \| v -  v_0 \|
    +
    \| \vHat_0  -  v_0 \| \\
    &\leq
    8 \sqrt{\| \theta^* \|^2 + \sigma^2} \sqrt{\frac{2d + \log(1/\delta)}{mn}} 
    +
   2 ( 3 \| \theta^* \| + \sigma) \exp(- 4 n) \\
   &\quad +
   (1 + \| \theta^* \| + \sigma) \exp \left(- n \cdot C_4(\alpha, \SNR ) \right)\\
    &=
     \sqrt{\| \theta^* \|^2 + \sigma^2} \sqrt{\frac{d + \log(1/\delta)}{mn}} \, \ccalO(1)
\end{align}
%%%%
with probability at least $1 - \delta - (m+2) \cdot 5^d \cdot \exp(-n/64) \geq 1 -2 \delta$. Here, $\ccalO(1)$ hides constants depending on $\| \theta^* \|$, $\sigma$ and $\alpha$. Also, we used the assumption that $n - 64 \log m \geq 104(2d + \log(1/\delta))$.

Moreover, we showed in Lemma \ref{lemma: gaussian covariance} that if the total number of samples in at least $mn \geq 192^2 (d + \log(2/\delta))$, then with probability $1 - \delta$, we have $\Vert \SigmaHat^{-1} \Vert_{\op} \leq 2$. This together with \eqref{eq: sup vHat - v} yields that with probability $1 - 3 \delta$, it holds that
%%%%
\begin{align}
    \sup_{\theta \in \mymathbb{B}(r;\theta^*)} T_1
    =
    \sup_{\theta \in \mymathbb{B}(r;\theta^*)} \| \SigmaHat^{-1} \|_{\op}  \| v - \vHat \|
    \leq
    \sqrt{\| \theta^* \|^2 + \sigma^2} \sqrt{\frac{d + \log(1/\delta)}{mn}}  \, \ccalO(1).
\end{align}
%%%%

%%%%%%%%%%%%%%%%%%%%%%%%%%%%%%%%%%%%%%%%%%%%%%%%%%%%%%%%%%%%%%%%
\subsubsection{Bounding $T_2$:}
%%%%%%%%%%%%%%%%%%%%%%%%%%%%%%%%%%%%%%%%%%%%%%%%%%%%%%%%%%%%%%%%

As we showed in Lemma \ref{lemma: gaussian covariance}, for a fixed $\delta \in (0,1)$ and $mn \geq 192^2 (d + \log(2/\delta))$, we have
%%%%
\begin{align}
    \Vert \SigmaHat^{-1} - I \Vert_{\op} \leq 192 \sqrt{\frac{d + \log(2/\delta)}{mn}}
\end{align}
%%%%
with probability $1 - \delta$. To bound $\Vert v \Vert$, we can write that $\| v \| = \| M(\theta) \| \leq \| M(\theta) - \theta^* \| + \| \theta^* \|$. The term $\| M(\theta) - \theta^* \|$ denotes the distance of the regression parameter $\theta$ to the optimal one $\theta^*$ after an iteration of updates by the population operator $M(\cdot)$. In Theorem \ref{thm: FOS}, we proved that this operator is contractive. More precisely, for any $\theta$ such that $\varepsilon \leq \| \theta - \theta^* \| \leq \alpha \| \theta^* \|$, we have
%%%%
\begin{align}
    \Vert M(\theta) - \theta^* \Vert
    \leq
    \kappa(\varepsilon) \cdot \Vert \theta - \theta^* \Vert,
    \text{ where }
    \kappa(\varepsilon)
    =
    \big(\| \theta^* \| + \sigma \big)
    \Big( \SNR + \frac{1}{n \varepsilon} \Big) \exp \left( -n \cdot C(\alpha, \SNR) \right).
\end{align}
%%%%
From the assumption of the theorem, we know that $\kappa(\varepsilon) \leq 1$ for large enough $n$.  Therefore  
%%%%
\begin{align}
    \| v \| = \| M(\theta) \| \leq \| M(\theta) - \theta^* \| + \| \theta^* \| \leq \| \theta - \theta^* \| + \| \theta^* \| \leq 2 \| \theta^* \|.
\end{align}
%%%%
All in all, we have with probability at least $1 - \delta$ that 
%%%%
\begin{align}
    T_2
    \leq
    \| \theta^* \| \sqrt{\frac{d + \log(1/\delta)}{mn}} \, \ccalO(1).
\end{align}
%%%%

Now having bounded both terms $T_1$ and $T_2$, we can write from \eqref{eq: Mm - M} that with probability at least $1 - 4 \delta$,
%%%%
\begin{align}
    \sup_{\theta \in \mymathbb{B}(r;\theta^*)}  \| M_m(\theta) - M(\theta) \|
    \leq
    \sqrt{\| \theta^* \|^2 + \sigma^2} \sqrt{\frac{d + \log(1/\delta)}{mn}}  \cdot \ccalO(1).
\end{align}
%%%%
We can further change the probability $1- 4 \delta$ to $1-\delta$ by replacing the $\log(1/\delta)$ to $\log(4/\delta)$ which implies slightly tighter bounds on the sample sizes $m$ and $n$. It is also worth noting that all the assumptions made in the auxiliary lemmas above can be implied by the ones made in Theorem \ref{thm: generalization}. Here, we conclude the proof of Theorem \ref{thm: generalization} and move to prove the auxiliary lemmas used above.

%%%%%%%%%%%%%%%%%%%%%%%%%%%%%%%%%%%%%%%%%%%%%%%%%%%%
\subsection{Proof of Lemma \ref{lemma: v - v0}} \label{proof: lemma v - v0}
%%%%%%%%%%%%%%%%%%%%%%%%%%%%%%%%%%%%%%%%%%%%%%%%%%%%
Using the definition of $v$ and $v_0$ in \eqref{eq: notation 1} and \eqref{eq: notation 2}, we have that
%%%%
\begin{align}
    \| v -  v_0 \|
    =
    \bigg\Vert \frac{1}{n} \E \left[ Z \right] - \frac{1}{n} \E \bigg[Z \tanh\bigg( \frac{1}{\sigma^2} \langle Z, \theta\rangle \bigg) \bigg] \bigg\Vert 
    \leq
    \frac{1}{n} \E \bigg[ \norm{Z} \bigg( 1 - \tanh\bigg( \frac{1}{\sigma^2} \langle Z, \theta\rangle \bigg)  \bigg)\bigg].
\end{align}
%%%%
Next, consider a fixed $\theta$ with $\Vert \theta - \theta^* \Vert \leq \alpha \Vert \theta^* \Vert$ and define a good event $\ccalE_1$ as follows
%%%%
\begin{align}
    \ccalE_1
    =
    \Big\{ \langle Z, \theta\rangle \geq \frac{n}{4} (1-\alpha) \| \theta^* \|^2 \Big\}.
\end{align}
%%%%
Therefore, we can write that
%%%%
\begin{align}
    \| v -  v_0 \|
    &\leq
    \frac{1}{n} \E \bigg[ \norm{Z} \bigg( 1 - \tanh\bigg( \frac{1}{\sigma^2} \langle Z, \theta\rangle \bigg)  \bigg) \cdot \mathbbm{1}\{\ccalE_1\}\bigg] \\
    &\quad +
    \frac{1}{n} \E \bigg[ \norm{Z} \bigg( 1 - \tanh\bigg( \frac{1}{\sigma^2} \langle Z, \theta\rangle \bigg)  \bigg) \cdot \mathbbm{1}\{\ccalE_1^c\}\bigg]
\end{align}
%%%%
Let us denote each of the two terms above as $T_3$ and $T_4$, that is,
%%%%
\begin{align} \label{eq: T3 T4}
    T_3
    &=
    \E \bigg[ \norm{Z} \bigg( 1 - \tanh\bigg( \frac{1}{\sigma^2} \langle Z, \theta\rangle \bigg)  \bigg) \cdot \mathbbm{1}\{\ccalE_1\}\bigg],\\
    T_4
    &=
    \E \bigg[ \norm{Z} \bigg( 1 - \tanh\bigg( \frac{1}{\sigma^2} \langle Z, \theta\rangle \bigg)  \bigg) \cdot \mathbbm{1}\{\ccalE_1^c\}\bigg].
\end{align}
%%%%
We bound $T_3$ by first noting that
%%%%
\begin{align}
    \norm{Z} \bigg( 1 - \tanh\bigg( \frac{1}{\sigma^2} \langle Z, \theta\rangle \bigg)  \bigg) \cdot \mathbbm{1}\{\ccalE_1\}
    \leq
    2 \norm{Z} \cdot \exp\left( - \frac{n}{2} (1-\alpha) \SNR^2 \right),
\end{align}
%%%%
where we used the fact that under $\ccalE_g$, we have $\langle Z, \theta\rangle \geq \frac{n}{4} (1-\alpha) \| \theta^* \|^2$ which from the monotonicity of $\tanh(\cdot)$ implies that 
%%%%
\begin{align}
    \tanh\bigg( \frac{1}{\sigma^2} \langle Z, \theta\rangle \bigg)
    \geq
    \tanh\bigg( \frac{n}{4} (1-\alpha) \frac{\| \theta^* \|^2}{\sigma^2} \bigg)
    \geq
    1
    -
    2 \exp\left( - \frac{n}{2} (1-\alpha) \SNR^2 \right).
\end{align}
%%%%
In the last inequality above, we used the fact that $\tanh(x) \geq 1 - 2 \exp(-2x)$ for all $x$. Consequently, we have that
%%%%
\begin{align}
    T_3
    \leq
    2 \E[\Vert Z \Vert] \cdot \exp\left( - \frac{n}{2} (1-\alpha) \SNR^2 \right).
\end{align}
%%%%
In the following, we upper bound $\E[\Vert Z \Vert]$. We can write that
%%%%
\begin{align}
    \E[\Vert Z \Vert^2]
    &=
    \E\bigg[ \bigg(\sum_{i=1}^{n} \langle \theta^*, X_i \rangle X_i + \epsilon_i X_i \bigg)^\top \bigg(\sum_{j=1}^{n} \langle \theta^*, X_j \rangle X_j + \epsilon_j X_j \bigg)\bigg] \\
    &=
    \sum_{i=1}^{n} \E\big[ Y_i^2 X_i^\top X_i \big]
    +
    \sum_{1 \leq i \neq j \leq n} \E\big[ Y_i X_i^\top X_j Y_j \big]\\
    &=
    (n^2 -n + d +2) \norm{\theta^*}^2 + d \sigma^2,
\end{align}
%%%%
which implies that $\E[\Vert Z \Vert] \leq \sqrt{n^2 -n + d +2} \Vert \theta^* \Vert + \sqrt{d} \sigma$ and consequently,
%%%%
\begin{align}
    T_3
    \leq
    2 \left( n \Vert \theta^* \Vert + \sqrt{d} \sigma \right) \cdot \exp\left( - \frac{n}{2} (1-\alpha) \SNR^2 \right).
\end{align}
%%%%

Next, we upper bound the term $T_4$ in \eqref{eq: T3 T4} as follows,
%%%%
\begin{align} \label{eq: integral}
    T_4
    &=
    \E \bigg[ \norm{Z} \bigg( 1 - \tanh\bigg( \frac{1}{\sigma^2} \langle Z, \theta\rangle \bigg)  \bigg) \cdot \mathbbm{1}\{\ccalE_1^c\}\bigg] \\
    & 
    \leq
    2 \E \left[ \norm{Z}  \cdot \mathbbm{1}\{\ccalE_1^c\}\right]\\
    & 
    =
    2 \int_{0}^{\infty} \Prob \left( \norm{Z}  \cdot \mathbbm{1}\{\ccalE_1^c\} \geq \gamma\right) \mathrm{d} \gamma.
\end{align}
%%%%
For any $\gamma \geq 0$, we use $\epsilon$-net argument and write that
%%%%
\begin{align} \label{eq: e-net}
    \Prob \left( \| Z \| \cdot \mathbbm{1}\{\ccalE_1^c\} \geq \gamma\right)
    &\leq
    \Prob \left( \| Z \| \geq \gamma\right) \\
    &=
    \Prob \bigg( \max_{\norm{u}=1} \langle Z,u \rangle \geq \gamma\bigg) \\
    &\leq
    \Prob \bigg( \max_{u' \in \ccalN_{1/2}} \langle Z,u' \rangle \geq \gamma/2 \bigg) \\
    &\leq
    5^d \cdot \Prob \left( \langle Z,u' \rangle \geq \gamma/2 \right),
\end{align}
%%%%
where $\ccalN_{1/2}$ denotes a $1/2$-covering of the unit sphere $\mathbb{S}^d=\{u \in \reals^d \, | \, \Vert u \Vert = 1\}$, known to have cardinality of at most $|\ccalN_{1/2}| \leq 5^d$. Moreover, in deriving the above inequalities, we used the fact that for any unit vector $u$, there exists $u' \in \ccalN_{1/2}$ such that $\Vert u - u' \Vert \leq 1/2$ and therefore,
%%%%
\begin{align}
    \langle Z,u \rangle
    =
    \langle Z, u - u' \rangle + \langle Z,u' \rangle
    \leq
    \max_{\Vert w \Vert =\frac{1}{2}} \langle Z, w \rangle 
    +
    \max_{u' \in \ccalN_{1/2}} \langle Z,u' \rangle,
\end{align}
%%%%
which yields that $\max_{\Vert u \Vert=1} \langle Z,u \rangle \leq 2 \max_{u' \in \ccalN_{1/2}} \langle Z,u' \rangle$. Now, consider a fixed unit vector $u' \in \mathbb{S}^d$. We know that $\langle Z,u' \rangle$ is $\SubE(8n \| \theta^* \|^2 + n \sigma^2, 4 \| \theta^* \| + \sigma)$ and $\E[\langle Z,u' \rangle] = n \langle \theta^*,u' \rangle$. Therefore,
%%%%
\begin{align}
    \Prob \left( \langle Z,u' \rangle \geq \gamma/2 \right)
    &=
    \Prob \left( \langle Z,u' \rangle - n \langle \theta^*,u' \rangle \geq \gamma/2 - n \langle \theta^*,u' \rangle \right) \\
    &\leq
    \Prob \left( \langle Z,u' \rangle - n \langle \theta^*,u' \rangle \geq \gamma/2 - n  \norm{\theta^*} \right) \\
    &\leq
    \Prob \left( | \langle Z,u' \rangle - n \langle \theta^*,u' \rangle | \geq \gamma/2 - n  \norm{\theta^*} \right) \\
    &\leq
    \exp\bigg( - \frac{1}{2}  \min\bigg\{ \frac{(\gamma/2 - n  \norm{\theta^*})^2}{8 n \| \theta^* \|^2 + n \sigma^2}, \frac{\gamma/2 - n  \norm{\theta^*}}{4 \| \theta^* \| + \sigma} \bigg\}\bigg).
\end{align}
%%%%
for any $\gamma \geq 2n  \norm{\theta^*}$. Now pick
%%%%
\begin{align}
    \gamma_0 = 2n  \norm{\theta^*} + 2 \frac{8 n \| \theta^* \|^2 + n \sigma^2}{4 \| \theta^* \| + \sigma} + 2 n^2 (\| \theta^* \| + \sigma).
\end{align}
%%%%
It yields that
%%%%
\begin{align} \label{eq: integral 1}
    \int_{\gamma_0}^{\infty} \Prob \left( \| Z \|  \cdot \mathbbm{1}\{\ccalE_1^c\} \geq \gamma\right) \mathrm{d} \gamma
    &\leq
    5^d \int_{\gamma_0}^{\infty} \exp\left( - \frac{\gamma/2 - n  \norm{\theta^*}}{4 \| \theta^* \| + \sigma} \right) \mathrm{d} \gamma\\
    &=
    5^d \cdot 2 (4 \| \theta^* \| + \sigma) \cdot \exp\left( - \frac{\gamma_0/2 - n  \norm{\theta^*}}{4 \| \theta^* \| + \sigma} \right)\\
    &\leq
    8 \cdot 5^d (\| \theta^* \| + \sigma) \exp(- n^2/4).
\end{align}
%%%%
%for a constant $c_2$ depending only on $\Vert \theta^* \Vert$ and $\sigma$.??
Moreover,
%%%%
\begin{align} \label{eq: integral 2}
    \int_{0}^{\gamma_0} \Prob \left( \| Z \|  \cdot \mathbbm{1}\{\ccalE_1^c\} \geq \gamma\right) \mathrm{d} \gamma
    &\leq
    \int_{e^{-c_3 n}}^{\gamma_0} \Prob \left( \| Z \|  \cdot \mathbbm{1}\{\ccalE_1^c\} \geq \gamma\right) \mathrm{d} \gamma
    +e^{-c_3 n} \\
    &\leq
    \gamma_0 \Prob \big( \ccalE_1^c \big) +e^{-c_3 n} \\
    &\leq
    (\gamma_0 + 1) \exp(-c_3 n) \\
    &\leq
    \Big(1 + (2n^2 + 10 n)(\| \theta^* \| + \sigma) \Big) \exp(-c_3 n),
\end{align}
%%%%
where we used Lemma \ref{lemma: events E1-3} to conclude that $ \Prob ( \ccalE_1^c ) \leq \exp(-c_3 n)$ for a constant $c_3$ depending on $\alpha$ and $\SNR$.
\iffalse
%%%%
\begin{lemma} \label{lemma: Prob good event}
For any $\theta \in \ccalB(\theta^*,r)$, we have that
%%%%
\begin{align}
    \langle Z, \theta\rangle \geq \frac{n}{4} (1-\alpha)
\end{align}
%%%%
with probability at least
%%%%
\begin{align}
    1 - 
    2 \exp\left( -\frac{n}{32} \left( \frac{1-\alpha}{1+\alpha}\right)^2\right)
    -
    2 \exp\left( -  n \min\left\{ \frac{1}{16} \left( \frac{1 - \alpha}{1 + \alpha}\right)^2 \SNR^2
    ,
    \frac{1}{8} \left( \frac{1 - \alpha}{1 + \alpha}\right) \SNR \right\}\right),
\end{align}
%%%%
\end{lemma}
%%%%

Proof of Lemma \ref{lemma: Prob good event}: Refer to Proof of Lemma \ref{lemma: event E-1 E-2} and \eqref{eq: Z(theta) positive}.
\fi
Putting \eqref{eq: integral 1} and \eqref{eq: integral 2} in \eqref{eq: integral} yields that
%%%%
\begin{align}
    T_4
    \leq
    2 \Big(1 + (2n^2 + 10 n)(\| \theta^* \| + \sigma) \Big) \exp(-c_3 n)
    +
    16 \cdot 5^d (\| \theta^* \| + \sigma) \exp(- n^2/4).
\end{align}
%%%%
Finally, we put everything together and conclude the lemma as follows,
%%%%
\begin{align}
    \| v -  v_0 \|
    &\leq
    \frac{1}{n} T_3
    +
    \frac{1}{n} T_4 \\
    &\leq
    2 \bigg( \Vert \theta^* \Vert + \frac{\sqrt{d}}{n} \sigma \bigg) \cdot \exp\left( - \frac{n}{2} (1-\alpha) \SNR^2 \right) \\
    &\quad+
    \frac{2}{n} \bigg( \Vert \theta^* \Vert + \frac{\sqrt{d}}{n} \sigma \bigg) \cdot \exp\left( - \frac{n}{2} (1-\alpha) \SNR^2 \right)
    +
    \frac{16}{n} \cdot 5^d (\| \theta^* \| + \sigma) \exp(- n^2/4)
    \\
    &\leq
    (1 + \| \theta^* \| + \sigma) \exp \left(- C_4(\alpha, \SNR ) n \right),
\end{align}
%%%%
for any $n \geq N_1(\alpha,\SNR)$. Here, we used the assumption that $d \leq n$ which is also implied by the assumptions in Theorem \ref{thm: generalization}, particularly from $n \geq 64 \log m + 104(2d + \log(4/\delta))$.
%%%%%%%%%%%%%%%%%%%%%%%%%%%%%%%%%%%%%%%%%%%%%%%%%%%%
\subsection{Proof of Lemma \ref{lemma: vHat0 - v0}} \label{proof: lemma vHat0 - v0}
%%%%%%%%%%%%%%%%%%%%%%%%%%%%%%%%%%%%%%%%%%%%%%%%%%%%

Recall from the notations that
%%%%
\begin{align}
    \| \vHat_0 -  v_0 \|
    =
    \bigg \Vert \frac{1}{mn} \sum_{j=1}^{m} Z^j - \frac{1}{n} \E \left[ Z \right] \bigg\Vert,
\end{align}
%%%%
For a fixed $j \in \{ 1,\cdots,m \}$ and unit-norm vector $u$, the inner product $\langle Z^u, u \rangle$ is sub-exponential. More precisely,
%%%%
\begin{align}
    \frac{1}{n} \langle  Z^j - \E[ Z^j],u \rangle
    \sim
    \SubE \Big(\frac{1}{n} (8\| \theta^* \|^2 + \sigma^2), \frac{1}{n}(4 \| \theta^* \| + \sigma) \Big).
\end{align}
%%%%
Therefore,
%%%%
\begin{align}
    \langle \vHat_0 -  v_0, u \rangle
    =
    \frac{1}{m} \sum_{j=1}^{m} 
    \frac{1}{n} \langle Z^j - \E[Z^j],u \rangle
    \sim
    \SubE \Big(\frac{1}{mn} (8\| \theta^* \|^2 + \sigma^2), \frac{1}{mn}(4 \| \theta^* \| + \sigma) \Big).
\end{align}
%%%%
From concentration of sub-exponential random variables in Theorem \ref{thm: sub-ex concentration}, we can write for every $t \geq 0$ that
%%%%
\begin{align}
    \Prob \left(  \| \vHat_0 -  v_0 \| \geq t \right)
    &=
    \Prob \bigg(  \max_{\norm{u}=1} \langle \vHat_0 -  v_0, u \rangle \geq t \bigg) \\
    &\leq
    \Prob \bigg(  \max_{u' \in \ccalN_{1/2}} \langle \vHat_0 -  v_0, u' \rangle \geq t/2 \bigg)\\
    &\leq
    5^d \cdot \exp\bigg( - \frac{1}{2}  \min\bigg\{ \frac{mn (t/2)^2}{8 \| \theta^* \|^2 +  \sigma^2}, \frac{mnt/2}{4 \| \theta^* \| + \sigma} \bigg\}\bigg)
\end{align}
%%%%
where we used a $1/2$-covering argument similar to \eqref{eq: e-net}. Now, assuming $mn \geq 32^2 (2d + \log(1/\delta))$, we pick
%%%%
\begin{align}
    t = 8 \sqrt{\frac{2d + \log(1/\delta)}{mn}} \sqrt{\| \theta^* \|^2 + \sigma^2},
\end{align}
%%%%
which implies that $\Prob (  \| \vHat_0 -  v_0 \| \geq t ) \leq \delta$ as desired.

%%%%%%%%%%%%%%%%%%%%%%%%%%%%%%%%%%%%%%%%%%%%%%%%%%%%
\subsection{Proof of Lemma \ref{lemma: vHat - vHat0}} \label{proof: lemma vHat - vHat0}
%%%%%%%%%%%%%%%%%%%%%%%%%%%%%%%%%%%%%%%%%%%%%%%%%%%%
We begin the proof by using the definitions of $\vHat$ and $\vHat_0$ and write
%%%%
\begin{align}
    \sup_{\theta \in \mymathbb{B}(r;\theta^*)} \| \vHat -  \vHat_0 \| 
    &=
    \sup_{\theta \in \mymathbb{B}(r;\theta^*)} \bigg\Vert \frac{1}{mn} \sum_{j=1}^{m} Z^j - \frac{1}{mn} \sum_{j=1}^{m} Z^j \tanh\Big( \frac{1}{\sigma^2} \langle Z^j, \theta\rangle \Big) \bigg\Vert \\
    &\leq
    \sup_{\theta \in \mymathbb{B}(r;\theta^*)} \frac{1}{mn} \sum_{j=1}^{m} \Vert Z^j \Vert \left(1 - \tanh\Big( \frac{1}{\sigma^2} \langle Z^j, \theta\rangle \Big) \right) \\
    &\leq
    \frac{1}{mn} \sum_{j=1}^{m} \Vert Z^j \Vert \cdot \sup_{\theta \in \mymathbb{B}(r;\theta^*)}\left(1 - \tanh\Big( \frac{1}{\sigma^2} \langle Z^j, \theta\rangle \Big) \right).
\end{align}
%%%%
Next, for each $j=1,\cdots,m$ we have that
%%%%
\begin{align} \label{eq: inf Z theta}
    \inf_{\theta \in \mymathbb{B}(r;\theta^*)} \langle Z^j, \theta\rangle 
    =
    \langle Z^j, \theta^* \rangle + \inf_{\theta \in \mymathbb{B}(r;\theta^*)} \langle Z^j, \theta - \theta^* \rangle  
    =
    \langle Z^j, \theta^* \rangle + r \cdot \inf_{\| u \| =1} \langle Z^j, u \rangle.
\end{align}
%%%%
Moreover, from a $1/2$-covering argument similar to \eqref{eq: e-net} we know that
%%%%
\begin{align}
    \sup_{\| u \| =1} \langle Z^j, u \rangle 
    \leq
    2 \sup_{u' \in \ccalN_{1/2}} \langle Z^j, u' \rangle,
\end{align}
%%%%
which yields that
%%%%
\begin{align}
    \Prob \bigg( \sup_{\| u \| =1} \langle Z^j, u \rangle \geq t \bigg)
    \leq
    \Prob \bigg( \sup_{u' \in \ccalN_{1/2}} \langle Z^j, u' \rangle \geq t/2 \bigg)
    \leq
    5^d \cdot \Prob \left( \langle Z^j, u' \rangle \geq t/2 \right).
\end{align}
%%%%
For a fixed unit-norm $u'$, $\langle Z^j,u' \rangle$ is $\SubE(8n \| \theta^* \|^2 + n \sigma^2, 4 \| \theta^* \| + \sigma)$ with $\E[\langle Z^j,u' \rangle] = n \langle \theta^*,u' \rangle$. Therefore, for any $t\geq 0$, we have that
%%%%
\begin{align}
    | \langle Z^j,u' \rangle - n \langle \theta^*,u' \rangle |
    \leq
    t/2,
\end{align}
%%%%
with probability at least 
%%%%
\begin{align}
    1
    -
    2 \exp\bigg( -  \min\bigg\{ \frac{t^2/8}{8n \| \theta^* \|^2 + n \sigma^2}, \frac{t/4}{4 \| \theta^* \| + \sigma} \bigg\}\bigg).
\end{align}
%%%%
Taking $t = n ( \| \theta^* \| + \sigma)$ in above yields that with probability at least $1 - \exp(-n/64)$ we have
%%%%
\begin{align}
    \langle Z^j,u' \rangle 
    \leq
    n \langle \theta^*,u' \rangle + \frac{1}{2} n ( \| \theta^* \| + \sigma)
    \leq
    \frac{3}{2} n \| \theta^* \|
    +
    \frac{1}{2} n \sigma.
\end{align}
%%%%
Together with the $1/2$-covering argument, it holds that
%%%%
\begin{align} \label{eq: sup Z u}
    \sup_{\| u \| =1} \langle Z^j, u \rangle 
    \leq
    3 n \| \theta^* \| + n \sigma,
\end{align}
%%%%
with probability at least $1 - 5^d \cdot \exp(-n/64)$. Moreover, following the above logic for a specific case of $u' = {\theta^*}/{\| \theta^* \|}$, we have with probability $1 - \exp(-n/64)$ that
%%%%
\begin{align}
    \Big| \langle Z^j, {\theta^*}/{\| \theta^* \|} \rangle - n \| \theta^* \| \Big|
    \leq
    \frac{1}{2} n ( \| \theta^* \| + \sigma),
\end{align}
%%%%
which implies that
%%%%
\begin{align} \label{eq: Z theta*}
    \langle Z^j, \theta^* \rangle
    \geq
    \frac{1}{2} n \| \theta^* \|^2
    -
    \frac{1}{2} n \| \theta^* \| \sigma,
\end{align}
%%%%
with the same probability. Putting \eqref{eq: sup Z u} and \eqref{eq: Z theta*} back in \eqref{eq: inf Z theta} yields that
%%%%
\begin{align}
    \inf_{\theta \in \mymathbb{B}(r;\theta^*)} \langle Z^j, \theta\rangle 
    =
    \langle Z^j, \theta^* \rangle + r \! \cdot \! \inf_{\| u \| =1} \langle Z^j, u \rangle 
    \geq
    \frac{1}{2} n \| \theta^* \|^2
    -
    \frac{1}{2} n \| \theta^* \| \sigma
    -
    3 r n \| \theta^* \| 
    -
    r n \sigma
    \geq
    \frac{1}{7} n \| \theta^* \|^2,
\end{align}
%%%%
for $r = \frac{1}{14} \| \theta^* \|$ and $\SNR \geq 4$. Therefore, with probability at least $1 - (5^d + 1)\exp(-n/64)$, we have
%%%%
\begin{align} \label{eq: sup 1 - tanh}
    \sup_{\theta \in \mymathbb{B}(r;\theta^*)}\bigg(1 - \tanh\Big( \frac{1}{\sigma^2} \langle Z^j, \theta\rangle \Big) \bigg)
    \leq
    2 \exp\Big( - \frac{2}{7} n \cdot \SNR^2 \Big)
    \leq
    2 \exp (- 4 n)
    ,
\end{align}
%%%%
where we used $\tanh(x) \geq 1 - 2 \exp(-2x)$ for all $x$. Furthermore, we showed above that with probability at least $1 - 5^d \cdot \exp(-n/64)$ and for each $1 \leq j \leq m$, we have
%%%%
\begin{align} \label{eq: norm Z}
    \| Z^j \|
    =
    \sup_{\| u \| =1} \langle Z^j, u \rangle 
    \leq
    3 n \| \theta^* \| + n \sigma.
\end{align}
%%%%
Putting \eqref{eq: sup 1 - tanh} and \eqref{eq: norm Z} together, we have
%%%%
\begin{align}
    \sup_{\theta \in \mymathbb{B}(r;\theta^*)} \| \vHat -  \vHat_0 \| 
    \leq
    \frac{1}{mn} \sum_{j=1}^{m} \Vert Z^j \Vert \cdot \sup_{\theta \in \mymathbb{B}(r;\theta^*)}\left(1 - \tanh\Big( \frac{1}{\sigma^2} \langle Z^j, \theta\rangle \Big) \right)
    \leq
    2 ( 3 \| \theta^* \| + \sigma) \exp(- 4 n),
\end{align}
%%%%
with probability at least $1 - (m+2) \cdot 5^d \cdot \exp(-n/64)$.

%%%%%%%%%%%%%%%%%%%%%%%%%%%%%%%%%%%%%%%%%%%%%%%%%%%%
\subsection{Proof of Lemma \ref{lemma: gaussian covariance}}
%%%%%%%%%%%%%%%%%%%%%%%%%%%%%%%%%%%%%%%%%%%%%%%%%%%%
The proof follows from basic standard Gaussian concentration. We provide the proof here for completeness. Using concentration of sub-exponential RVs and $\epsilon$-net arguments we have that
%%%%
\begin{align}
    \Prob \left( \big\Vert \SigmaHat - I_d \big\Vert_{\op} \geq t\right)
    \leq
    2 \cdot 9^d \exp\bigg( -\frac{n}{2} \min\bigg\{\bigg(\frac{t}{32}\bigg)^2, \frac{t}{32} \bigg\}\bigg).
\end{align}
%%%%
Picking $t = 96 \sqrt{(d + \log(2/\delta))/N}$ yields the desired concentration bound. For the second inequality, we note that
%%%%
\begin{align}
    \big\Vert \SigmaHat^{-1} - I_d \big\Vert_{\op}
    &\leq
    \big\Vert \SigmaHat^{-1} \big\Vert_{\op} \big\Vert \SigmaHat - I_d \big\Vert_{\op}\\
    &\leq
    \left( 1 + \big\Vert \SigmaHat^{-1} - I_d \big\Vert_{\op} \right)
    \big\Vert \SigmaHat - I_d \big\Vert_{\op}\\
    &\leq
    96 \sqrt{\frac{d + \log(2/\delta)}{N}}
    +
    \frac{1}{2} \big\Vert \SigmaHat^{-1} - I_d \big\Vert_{\op},
\end{align}
%%%%
with probability at least $1 - \delta$. In above, we used the concentration proved in the first part, as well as the assumption $N \geq 192^2 (d + \log(2/\delta))$ to conclude that $\Vert \SigmaHat - I_d \Vert_{\op} \leq 1/2$.

%% file: 6-useful_stuff.tex
%%%%%%%%%%%%%%%%%%%%%%%%%%%%%%%%%%%%%%%%%%%%%%%%%%%%
\subsection{Proof of Proposition \ref{prop: population EM}} \label{proof: prop population EM}
%%%%%%%%%%%%%%%%%%%%%%%%%%%%%%%%%%%%%%%%%%%%%%%%%%%%

Let us denote $x_{[n]} \coloneqq (x_1\cdots,x_n)$ and $y_{[n]} \coloneqq (y_1\cdots,y_n)$. Then, according to the C-MLR model in \eqref{eq: pop-C-MLR model} with true regression parameters $\theta$, we have $y_i | \xi, x_i \sim \ccalN(\xi \langle x_i, \theta \rangle, \sigma^2)$. Therefore,
%%%%
\begin{align}
    f_{\theta} (x_{[n]},y_{[n]})
    &=
    \int  f_{\theta} (x_{[n]},y_{[n]},\xi) \mathrm{d} \xi \\
    &=
    \Prob(\xi = -1) f_{\theta} (x_{[n]},y_{[n]} | \xi=-1) + \Prob(\xi = +1) f_{\theta} (x_{[n]},y_{[n]} | \xi=+1)\\
    &=
    \frac{1}{2} \bigg( \frac{1}{2 \pi \sigma^2} \bigg)^{n/2} f(x_{[n]}) 
    \bigg\{ 
    \exp\bigg( - \frac{1}{2\sigma^2} \sum_{i=1}^{n} \big( y_i - \langle x_i, \theta \rangle \big)^2\bigg)
    +
    \exp\bigg( - \frac{1}{2\sigma^2} \sum_{i=1}^{n} \big( y_i + \langle x_i, \theta \rangle \big)^2\bigg)
    \bigg\} \\
    &=
    \frac{1}{2} \bigg( \frac{1}{2 \pi \sigma^2} \bigg)^{n/2} f(x_{[n]}) 
    \left( \exp(s_{-1}({\theta})) + \exp(s_{+1}({\theta})) \right),
\end{align}
%%%%
where we denote for any $\theta$ and $\xi \in \{-1,+1\}$
%%%%
\begin{align}
    s_{\xi}({\theta})
    =
    - \frac{1}{2\sigma^2} \sum_{i=1}^{n} \big( y_i - \xi \langle x_i, \theta \rangle \big)^2.
\end{align}
%%%%
Moreover,
%%%%
\begin{align}
    f_{\theta} (\xi | x_{[n]},y_{[n]})
    &=
    \frac{f_{\theta} (x_{[n]},y_{[n]},\xi)}{f_{\theta} (x_{[n]},y_{[n]})}
    =
    2 p(\xi) \frac{\exp(s_{\xi}({\theta})))}{\exp(s_{-1}({\theta})) + \exp(s_{+1}({\theta}))},
\end{align}
%%%%
and
%%%%
\begin{align} \label{eq: ln 3 terms}
    \log f_{\theta'} (x_{[n]},y_{[n]},\xi)
    =
    s_{\xi}({\theta'}) + \log p(\xi) + \log f(x_{[n]})
    -
    \frac{n}{2} \log(2 \pi).
\end{align}
%%%%
Let us define
%%%%
\begin{align}
    \hat{Q}(\theta' | \theta)
    =
    \int_{\xi} f_{\theta} (\xi | x_{[n]},y_{[n]}) \log f_{\theta'} (x_{[n]},y_{[n]},\xi)  \mathrm{d} \xi.
\end{align}
%%%%
According to \eqref{eq: ln 3 terms}, only the first term in the RHS of \eqref{eq: ln 3 terms} depends on $\theta'$ and therefore, we only keep the term $s_{\xi}({\theta'})$ in computing $\hat{Q}(\theta' | \theta)$ as follows
%%%%
\begin{align} \label{eq: int term 1}
    \int_{\xi} f_{\theta} (\xi | x_{[n]},y_{[n]}) s_{\xi}({\theta'}) \mathrm{d} \xi
    &=
    \int_{\xi} 2 p(\xi) \frac{\exp(s_{\xi}({\theta})))}{\exp(s_{-1}({\theta})) + \exp(s_{+1}({\theta}))} s_{\xi}({\theta'}) \mathrm{d} \xi \\
    &=
    - \frac{1}{2\sigma^2}  \sum_{i=1}^{n} y_i^2
    -
    \frac{1}{2\sigma^2}  \sum_{i=1}^{n} \langle x_i, \theta' \rangle^2\\
    &\quad
    +
    \frac{1}{\sigma^2} \frac{\exp(s_{+1}({\theta})) - \exp(s_{-1}({\theta}))}{\exp(s_{+1}({\theta})) + \exp(s_{-1}({\theta}))}  \sum_{i=1}^{n} y_i \langle x_i, \theta' \rangle.
\end{align}
%%%%
Now, the $Q$-function defined in \eqref{eq: pop-EM} can be written as $Q(\theta' | \theta) = \E[\hat{Q}(\theta' | \theta)]$ where the expectation is w.r.t. randomness in $(X_{[n]},Y_{[n]})$ generated by the ground truth distribution governed by $\theta^*$. From \eqref{eq: int term 1} we have that
%%%%
\begin{align} \label{eq: E 3 terms}
    \E \bigg[\int_{\xi} f_{\theta} (\xi | x_{[n]},y_{[n]}) s_{\xi}({\theta'}) \mathrm{d} \xi \bigg]
    &=
    - \frac{1}{2\sigma^2}  \sum_{i=1}^{n} \E [ y_i^2 ]
    -
    \frac{n}{2\sigma^2} \Vert \theta' \Vert^2 \\
    &\quad
    +
    \frac{1}{\sigma^2} \E \bigg[ \tanh \bigg( \frac{1}{\sigma^2} \sum_{i=1}^{n} y_i \langle x_i, \theta \rangle \bigg)
    \sum_{i=1}^{n} y_i \langle x_i, \theta' \rangle \bigg],
\end{align}
%%%%
where we used the fact that
%%%%
\begin{align}
    \frac{\exp(s_{+1}({\theta})) - \exp(s_{-1}({\theta}))}{\exp(s_{+1}({\theta})) + \exp(s_{-1}({\theta}))} 
    =
    \tanh \bigg( \frac{1}{\sigma^2} \sum_{i=1}^{n} y_i \langle x_i, \theta \rangle \bigg).
\end{align}
%%%%
Note that the first term in RHS of \eqref{eq: E 3 terms} dose not depend on $\theta'$, therefore,
%%%%
\begin{align}
    \gr Q(\theta' | \theta)
    =
    - \frac{n}{\sigma^2} \theta' 
    +
    \frac{1}{\sigma^2} \E\bigg[ \sum_{i=1}^{n} X_i Y_i \tanh\bigg( \frac{1}{\sigma^2} \sum_{i=1}^{n} \langle X_i, \theta\rangle Y_i \bigg) \bigg].
\end{align}
%%%%
Putting $ \gr Q(\theta' | \theta) = 0$ yields that
%%%%
\begin{align}
    M(\theta)
    &\coloneqq
    \argmax_{\theta'} Q(\theta' | \theta)\\
    &
    =
    \frac{1}{n} \E\bigg[ \sum_{i=1}^{n} X_i Y_i \tanh\bigg( \frac{1}{\sigma^2} \sum_{i=1}^{n} \langle X_i, \theta\rangle Y_i \bigg) \bigg] \\
    &=
    \E\bigg[ X_1 Y_1 \tanh\bigg( \frac{1}{\sigma^2} \sum_{i=1}^{n} \langle X_i, \theta\rangle Y_i \bigg) \bigg].
\end{align}
%%%%

%%%%%%%%%%%%%%%%%%%%%%%%%%%%%%%%%%%%%%%%%%%%%%%%%%%%
\subsection{Proof of Proposition \ref{prop: empirical EM}} \label{proof: prop populatiempiricalon EM}
%%%%%%%%%%%%%%%%%%%%%%%%%%%%%%%%%%%%%%%%%%%%%%%%%%%%

Using \eqref{eq: int term 1} and definition of $Q_m$-function in \eqref{eq: emp Q} we have that
%%%%
\begin{align} 
    \frac{1}{m} \sum_{j=1}^{m} \int_{\xi} f_{\theta} (\xi | x_{[n]}^j,y_{[n]}^j) s_{\xi}({\theta'}) \mathrm{d} \xi
    &=
    - \frac{1}{2m\sigma^2}  \sum_{j=1}^{m} \sum_{i=1}^{n} {y_i^j}^2
    -
    \frac{1}{2m\sigma^2}  \sum_{j=1}^{m} \sum_{i=1}^{n} \langle x_i^j, \theta' \rangle^2 \\
    &\quad+
    \frac{1}{m\sigma^2} \sum_{j=1}^{m} \tanh \bigg( \frac{1}{\sigma^2} \sum_{i=1}^{n} y_i^j \langle x_i^j, \theta \rangle \bigg)  \sum_{i=1}^{n} y_i^j \langle x_i^j, \theta' \rangle.
\end{align}
%%%%
Therefore,
%%%%
\begin{align}
    \gr Q_m(\theta' | \theta)
    =
    - \frac{1}{m\sigma^2}\sum_{j=1}^{m} \sum_{i=1}^{n} x_i^j {x_i^j}^{\top} \theta' 
    +
    \frac{1}{m\sigma^2} \sum_{j=1}^{m} \tanh \bigg( \frac{1}{\sigma^2} \sum_{i=1}^{n} y_i^j \langle x_i^j, \theta \rangle \bigg)  \sum_{i=1}^{n} y_i^j  x_i^j,
\end{align}
%%%%
and finally putting $\gr Q_m(\theta' | \theta) = 0$ yields the desired result.

%%%%
\begin{definition}[Sub-exponential RV] \label{def: sub-exponential}
A random variable $X$ is said to be sub-exponential with parameter $(\tau^2,b)$ if
%%%%
\begin{align}
    \E \left[\exp \left(\lambda (X - \mu_X) \right) \right] 
    \leq
    \exp\bigg(\frac{\lambda^2 \tau^2}{2}\bigg), 
    \quad
    \forall |\lambda| \leq \frac{1}{b}.
\end{align}
%%%%
We denote such RV by $\SubE(\tau^2,b)$.
\end{definition}
%%%%

%%%%
\begin{theorem} \label{thm: sub-ex concentration}
Let $X_i \sim \SubE(\tau^2, b)$ be iid sub-exponentials. Then,
%%%%
\begin{align}
    \Prob\bigg( \bigg| \frac{1}{n} \sum_{i=1}^{n} X_i - \mu_X \bigg| \geq t \bigg)
    \leq
    2 \exp\bigg( -  \min\bigg\{ \frac{nt^2}{2\tau^2}, \frac{nt}{2b} \bigg\}\bigg)
\end{align}
%%%%
\end{theorem}
%%%%

\begin{lemma} \label{lemma: tanh slope}
For any $x_1, x_2 \geq 0$, we have
%%%%
\begin{align}
    \frac{\tanh(x_2) - \tanh(x_1)}{x_2 - x_1}
    \leq
    \max\{1-\tanh^2(x_1), 1 - \tanh^2(x_2)\}.
\end{align}
%%%%
\end{lemma}
%%%%

\begin{proof}
Assume that $x_2 \geq x_1 \geq 0$. Since the function $f(x) \coloneqq \tanh(x)$ is concave in $[0,+\infty)$, we can write
%%%%
\begin{align}
    f(x_2) \leq f(x_1) + (x_2 - x_1) f'(x_1),
\end{align}
%%%%
which yields that
%%%%
\begin{align}
    \frac{\tanh(x_2) - \tanh(x_1)}{x_2 - x_1}
    \leq
    1-\tanh^2(x_1).
\end{align}
%%%%
Similar argument holds for $x_1 \geq x_2 \geq 0$, i.e.,
%%%%
\begin{align}
    \frac{\tanh(x_1) - \tanh(x_2)}{x_1 - x_2}
    \leq
    1-\tanh^2(x_2).
\end{align}
%%%%
\end{proof}
%%%%